\documentclass[preprint]{imsart}

\usepackage[utf8]{inputenc}
\usepackage[T1]{fontenc}
\usepackage[numbers,comma,sort&compress]{natbib}
\usepackage{xspace}
\usepackage[francais,english]{babel}
\usepackage{amsmath,amssymb,amsthm,bm}
\usepackage{geometry}
\usepackage{enumerate,url}
\usepackage[colorlinks,citecolor=blue,urlcolor=blue]{hyperref}
\usepackage{scalerel}
\usepackage{ulem} 

\arxiv{1608.01118v3}

\numberwithin{equation}{section}

\theoremstyle{plain}
\newtheorem{theorem}{Theorem}[section]
\newtheorem{proposition}[theorem]{Proposition}
\newtheorem{lemma}[theorem]{Lemma}
\newtheorem{corollary}[theorem]{Corollary}
\newtheorem{definition}[theorem]{Definition}
\newtheorem{stass}[theorem]{Standing assumptions}
\theoremstyle{remark}
\newtheorem{remark}[theorem]{Remark}

\newcounter{enumiSave}

\DeclareMathOperator \argmin {argmin}
\DeclareMathOperator \argmax {argmax}

\DeclareMathOperator \var    {var}
\DeclareMathOperator \cov    {cov}
\DeclareMathOperator \Cond   {Cond}

\newcommand \Esp  {\mathrm{E}}
\newcommand \Prob {\mathrm{P}}
\newcommand \Qrob {\mathrm{Q}}
\newcommand \as   {\mathrm{a.s.}}

\usepackage[centercolon = true]{mathtools}
\newcommand \eqdef {\mathrel{\coloneqq}}

\newcommand \eqas {\stackrel{\text{a.s.}}{=}}

\newcommand \toas {\xrightarrow[]{\as}}

\usepackage{dsfont}             
\newcommand \one  {\mathds{1}}  
\newcommand \Eset {\mathbb{E}}  
\newcommand \Nset {\mathbb{N}}  
\newcommand \Rset {\mathbb{R}}  
\newcommand \Xset {\mathbb{X}}  
\newcommand \Dset {\mathbb{D}}  
\newcommand \Sset {\mathbb{S}}  
\newcommand \Zset {\mathbb{Z}}  
\newcommand \Mset {\mathbb{M}}  

\newcommand \Rplus    {\left[ 0, +\infty \right)}
\newcommand \RplusExt {\left[ 0, +\infty \right]}

\newcommand \Bcal {\mathcal{B}}  
\newcommand \Ccal {\mathcal{C}}  
\newcommand \Ecal {\mathcal{E}}  
\newcommand \Fcal {\mathcal{F}}  
\newcommand \Gcal {\mathcal{G}}  
\newcommand \Hcal {\mathcal{H}}  
\newcommand \Jcal {\mathcal{J}}  
\newcommand \Mcal {\mathcal{M}}  
\newcommand \Ncal {\mathcal{N}}  
\newcommand \Scal {\mathcal{S}}  
\newcommand \Tcal {\mathcal{T}}  
\newcommand \GP   {\mathcal{GP}} 
\newcommand \GM   {\GP}          

\newcommand \CcalX    {\mathcal{C}(\Xset)}  
\newcommand \MGCX     {\Mset}               

\newcommand \McalGCX  {\Mcal}               
\newcommand \Pfrak    {\mathfrak{P}}        
\newcommand \MsetQcq  {\Mset_0}
\newcommand \McalQcq  {\Mcal_0}

\newcommand \Pn   {\Prob_n}
\newcommand \Pnxi {\Pn^\xi}
\newcommand \Pinfxi {\Prob_{\infty}^\xi}
\newcommand \PxiParen[1] {\bigl( \Prob_{#1}^\xi \bigr)}
\newcommand \nug  {{\scaleobj{1.1}{\bm{\nu}}}}
\newcommand \onug {\overline{\nug}}

\newcommand \ddiff  {\mathrm{d}}
\newcommand \dmu    {\ddiff \mu}
\newcommand \dnu    {\ddiff \nu}
\newcommand \du     {\ddiff u}
\newcommand \dv     {\ddiff v}

\newcommand \dx     {\ddiff x}

\newcommand \df     {\ddiff f}

\newcommand \xx {\underline x}%
\newcommand \zz {\underline z}%
\newcommand \XX {\underline X}%
\newcommand \ZZ {\underline Z}%

\newcommand \UnitCube {\left[ 0, 1 \right]^\ell}  

\newcommand \AscFgcd {$\Pfrak$-continuous\xspace}
\newcommand \UiFgcd  {$\Pfrak$-uniformly integrable\xspace}
\newcommand \RCS     {random closed set\xspace}   
\newcommand \RCSs    {random closed sets\xspace}  

\begin{document}

\begin{frontmatter}

  \title{A supermartingale approach to Gaussian process based
    sequential design of experiments}

  \runtitle{Supermartingale approach to Gaussian process design of
    experiments}

  \begin{aug}
    \author{\fnms{Julien} \snm{Bect}\thanksref{m1,a}%
      \ead[label=e1]{julien.bect@centralesupelec.fr}},
    \author{\fnms{Fran\c{c}ois} \snm{Bachoc}\thanksref{b}%
      \ead[label=e2]{francois.bachoc@math.univ-toulouse.fr}}
    \and
    \author{\fnms{David} \snm{Ginsbourger}\thanksref{c,d}%
      \ead[label=e3]{ginsbourger@idiap.ch}
      \ead[label=e4]{ginsbourger@stat.unibe.ch}}

    \runauthor{J.~Bect, F.~Bachoc and D.~Ginsbourger}

    \address[a]{%
      Laboratoire des Signaux et Syst\`emes (L2S)\\
      CentraleSup\'elec, CNRS, Univ. Paris-Sud, Universit\'e Paris-Saclay,
      Gif-sur-Yvette, France.\\
      \printead{e1}
    }
    \address[b]{%
      Toulouse Mathematics Institute,
      University Paul Sabatier, France.\\
      \printead{e2}
    }
    \address[c]{%
      Uncertainty Quantification and Optimal Design group,\\
      Idiap Research Institute, Martigny, Switzerland\\
      \printead{e3}
    }
    \address[d]{%
      Institute of Mathematical Statistics and Actuarial Science,\\
      Department of Mathematics and Statistics,
      University of Bern, Switzerland\\
      \printead{e4}
    }
    \address{\thanksmark{m1}%
      Corresponding author.
    }

  \end{aug}

  \begin{abstract}\textbf{. }
    Gaussian process (GP) models have become a well-established
    framework for the adaptive design of costly experiments, and
    notably of computer experiments.
    GP-based sequential designs have been found practically efficient
    for various objectives, such as global optimization (estimating
    the global maximum or maximizer(s) of a function), reliability
    analysis (estimating a probability of failure) or the estimation
    of level sets and excursion sets.
    In this paper, we study the consistency of an important class of
    sequential designs, known as stepwise uncertainty reduction (SUR)
    strategies.
    Our approach relies on the key observation that the sequence of
    residual uncertainty measures, in SUR strategies, is generally a
    supermartingale with respect to the filtration generated by the
    observations.
    This observation enables us to establish generic consistency
    results for a broad class of SUR strategies.
    The consistency of several popular sequential design strategies is
    then obtained by means of this general result.
    Notably, we establish the consistency of two SUR strategies
    proposed by Bect, Ginsbourger, Li, Picheny and Vazquez
    (Stat.~Comp., 2012)---to the best of our knowledge, these are the
    first proofs of consistency for GP-based sequential design
    algorithms dedicated to the estimation of excursion sets and
    their measure.
    We also establish a new, more general proof of consistency for the
    expected improvement algorithm for global optimization which,
    unlike previous results in the literature, applies to any GP with
    continuous sample paths.
\end{abstract}



\end{frontmatter}

\section{Introduction}

Sequential design of experiments is an important and lively research
field at a crossroads between applied probability, statistics and
optimization, where the goal is to allocate experimental resources
step by step so as to reduce the uncertainty about some quantity, or
function, of interest.
While the experimental design vocabulary traditionally refers to
observations of natural phenomena presenting aleatory uncertainties,
the design of computer experiments---in which observations are
replaced by numerical simulations---has become a field of research
\textit{per se} \citep{DACE, Koehler.etal1998, TDACE}, where Gaussian
process models are massively used to define efficient sequential
designs in cases of costly evaluations.
The predominance of Gaussian processes in this field is probably due
to their unique combination of modeling flexibility and computational
tractability, which makes it possible to work out sampling criteria
accounting for the potential effect of adding new experiments.
The definition, calculation and optimization of sampling criteria
tailored to various application goals have inspired a significant
number of research contributions in the last decades \cite[see,
e.g.,][]{frazier08, bect12, chevalier13, chevalier14, cohn96al,
  Ginsbourger.etal2014, Shahriari2016, Feliot.etal, Scott.etal2011,
  Frazier2009, frazier09, Gramacy2016AL, Ranjan2008, villemonteix08,
  Picheny2014, Picheny.etal2010}.
Yet, available convergence results for the associated sequential designs
are quite heterogeneous in terms of their respective extent
and underlying hypotheses \cite{vazquez2010convergence,
  Grunewalder.etal2010, Scott.etal2011, bull11, Srinivas.etal2012}.
Here we develop a probabilistic approach to the analysis of a large
class of strategies. This enables us to establish generic consistency
results, whose broad applicability is subsequently illustrated on four
popular sequential design strategies.
The crux is that each of these strategies turns out to involve some
\textit{uncertainty functional} applied to a sequence of conditional
probability distributions, and our main results rely on %
the key property---which will be referred to as the \emph{supermartingale
  property}---that, for any sequential design, the sequence of random
variables produced by these functionals is a supermartingale with
respect to the filtration generated by the observations.

Among the sampling criteria considered in our examples, probably the
most famous one is the expected improvement (EI), that arose in
sequential design for global optimization.
Following the foundations laid by \citet{mockus78} and the
considerable impact of the work of~\citet{EGOEBBF}, EI and other
Bayesian optimization strategies have spread in a variety of
application fields.
They are now commonly used in engineering design \citep{forrester08}
and, in the field of machine learning, for automatic configuration
algorithms (see~\cite{Shahriari2016} and references therein).
Extensions to constrained, multi-objective and/or robust optimization
constitute an active field of research \citep[see,
e.g.,][]{Williams.etal2000, Emmerich.etal2006, Picheny2014,
  Binois2015, Gramacy2016AL, Feliot.etal}.
In a different context, sequential design strategies based on Gaussian
process models have been used to estimate contour lines, probabilities
of failures, profile optima and excursion sets of expensive to
evaluate simulators \citep[see, notably,][]{Ranjan2008,
  Vazquez.Bect2009, Picheny.etal2010, bect12, Zuluaga.etal2013,
  chevalier14, Ginsbourger.etal2014, Wang.etal2016}.

More specifically, we consider in this paper sequential design strategies
built according to the \emph{stepwise uncertainty reduction} (SUR)
paradigm \citep[see][and references therein]{villemonteix08, bect12,
  chevalier13}.
Our main focus is the consistency of these algorithms under the
assumption that the function of interest is a sample path of the
Gaussian process model that is used to construct the sequential
design.
Almost sure consistency has been proved for the EI algorithm
in~\citep{vazquez2010convergence}, but only under the restrictive
assumption that the covariance function satisfies a certain
condition---the ``No Empty Ball'' (NEB) property%
---which excludes very regular Gaussian processes\footnote{%
  On a related note, \citet{bull11} proves an upper-bound for the
  convergence rate of the expected improvement algorithm under the
  assumption that the covariance function is Hölder, but his result
  only holds for functions that belong to the reproducing kernel
  Hilbert space (RKHS) of the covariance---a condition which, under
  appropriate assumptions, is almost surely not satisfied by sample
  paths of the Gaussian process according to Driscoll's theorem
  \citep{lukic01}.  Another result in the same vein is provided
  by~\citet{yarotsky13} for the squared exponential covariance in the
  univariate case, assuming that the objective function is analytic in
  a sufficiently large complex domain around its interval of
  definition.}.
Moreover, to the authors' knowledge, no proof of consistency has yet
been established for algorithms dedicated to probability of excursion
and/or excursion set estimation (referred to as \textit{excursion
  case} henceforth) such as those of~\citet{bect12}.
The scheme of proof developed in this work allows
us to address the excursion case and also to revisit the consistency
of the knowledge gradient algorithm~\cite{frazier08, frazier09,
  Frazier2009}, as well as that of the EI algorithm---which can
also be seen as a particular case of a SUR strategy
\cite{chevalier13}---without requiring the NEB assumption.
Before outlining the paper in more detail, let us briefly introduce
the general setting and, in particular, what we mean by SUR
strategies.
We will focus directly on the case of Gaussian processes for clarity,
but the SUR principle in itself is much more general, and can be used
with other types of models %
\citep[see, e.g.,][]{johnson1960, ogeran1993, mackay92ibof, cohn96al,
  geman96at, kingsmith1994}.%

Let $\xi$ be a real-valued Gaussian process defined on a measurable
space~$\Xset$---typically, $\xi$ will be a continuous Gaussian process
on a compact metric space, such as $\Xset = \UnitCube$---
and assume that evaluations (observations)
$Z_n = \xi(X_n) + \epsilon_n$ are to be made, sequentially, in order
to estimate some quantity of interest (e.g., the maximum of~$\xi$, or
the volume of its excursion above some given threshold).
We will assume the sequence of observation errors
$(\epsilon_n)_{n \in \mathbb{N}^*}$ to be independent of the Gaussian
process $\xi$, and composed of independent centered Gaussian
variables.
The definition of a SUR strategy starts with the choice of a ``measure
of residual uncertainty'' for the quantity of interest after
$n$~evaluations, which is a functional
\begin{equation}
  H_n = \Hcal \left( \Pnxi \right)
\end{equation}
of the conditional distribution~$\Pnxi$ of~$\xi$ given~$\Fcal_n$,
where $\Fcal_n$ is the $\sigma$-algebra generated by~$X_1$, $Z_1$,
\ldots, $X_n$, $Z_n$.
Assuming that the $H_n$'s are $\Fcal_n$-measurable random variables,
a SUR sampling criterion is then defined as
\begin{equation}
  \label{eq:generic-SUR:sampcrit-Jn}
  J_n(x) \;=\; \Esp_{n,x} \left( H_{n+1} \right),
\end{equation}
where $\Esp_{n,x}$ denotes the conditional expectation with respect
to~$\Fcal_n$ with~$X_{n+1} = x$ (assuming that~$H_{n+1}$ is
integrable, for any choice of~$x \in \Xset$).
The value of the sampling criterion~$J_n(x)$ at time~$n$ quantifies
the expected residual uncertainty at time $n + 1$ if the next
evaluation is made at~$x$.
Finally, a (non-randomized) sequential design is constructed by
greedily choosing at each step a point that provides the smallest
expected residual uncertainty---equivalently, the largest expected
uncertainty reduction---that is,
\begin{equation}
  \label{eq:generic-SUR:Xnplusone}
  X_{n+1} \in \argmin_{x \in \Xset}\, J_n(x).
\end{equation}
Such a greedy strategy is sometimes called myopic or one-step
look-ahead (as opposed to a Bayes-optimal strategy, which would
consider the reduction of uncertainty achieved at the end of the
sequential design, i.e., when the entire experimental budget has been
spent).
Our goal is to establish the consistency of these strategies, where
consistency means that the residual uncertainty~$H_n$ goes almost
surely to zero.

Given a finite measure $\mu$ over $\Xset$ and an excursion threshold
$T \in \Rset$, a typical choice of measure of residual uncertainty in
the excursion case \citep{bect12} is the integrated indicator variance
$ H_n = \Hcal \left( \Pnxi \right)=\int_{\Xset} p_{n} \left( 1 - p_{n}
\right) \dmu$ (also called \textit{integrated Bernoulli variance} in
what follows) where $p_n(x) = \Prob_n \left( \xi(x) \geq T \right)$
and $\Prob_n$ denotes the conditional probability with respect
to~$\Fcal_n$.
Note that $p_{n}(x)(1-p_{n}(x))=\var_n(\one_{\xi(x) \geq T})$, where
$\var_n$ denotes the conditional variance with respect to~$\Fcal_n$.
Recalling the definition of~$J_n$
from~\eqref{eq:generic-SUR:sampcrit-Jn}, and using the law of total
variance, we obtain that
\begin{equation}
  J_{n}(x) = \int_{\Xset} \Esp_{n,x} \left(
    \var_{n+1}\left( \one_{\xi(u) \geq T} \right)
  \right) \mu(\du) \leq \int_{\Xset}
  \var_{n}\left( \one_{\xi(u) \geq T} \right)
  \mu(\du) = H_n,
\end{equation}
which shows that $(H_{n})$ is an
$\left( \Fcal_n \right)$-supermartingale.
Another related measure of uncertainty, for which a semi-analytical
formula is provided in~\cite{chevalier14}, is the \textit{variance of
  the excursion volume},
$H_n = \var_n ( \mu(\{x \in \Xset: \xi(x) \geq T\}) )$.
The supermartingale property follows again in this case from the law
of total variance.
In the optimization case, on the other hand, it turns out \citep[see,
e.g.,][Section~3.3]{chevalier13} that the EI criterion is underlaid by
the measure of residual uncertainty
$H_n = \Esp_n \left( \max\xi - M_n \right)$, where $M_n$ is defined as
$M_n = \max_{i \le n} \xi(X_i)$ for non-degenerate Gaussian processes
(see Remark~\ref{rem:unusual-EI}), and $\Esp_n$ denotes the
conditional expectation with respect to~$\Fcal_n$.
A similar construction can be obtained for the \textit{knowledge
  gradient}, as developed later.
It turns out as shown later in the paper that, for both criteria, the
associated measures of residual uncertainty also possess the
aforementioned supermartingale property.

It must be pointed out here that, under very weak assumptions about
the Gaussian process model and the uncertainty functional, consistency
can also be achieved more simply using any dense (``space filling'')
deterministic sequence of design points.
It has been largely demonstrated, however, that SUR strategies
typically outperform in practice these simple deterministic designs
(see references above).
Hence there remains a gap between theory and practice that is not
filled by our consistency results since, by themselves, they do not
provide a very strong theoretical support for the choice of SUR
sequential designs over other types of designs, and in particular over
non-sequential designs.

The main practical interest of our consistency results is rather to
answer the natural concern that SUR strategies, because of their
greedy nature, might fail to be consistent in some situations.
Such a concern is justified for instance by the explicit
counterexample, provided by Yarotsky \cite{Yarotsky2013}, of a
particular function for which the sequence of points generated by the
EI strategy fails to produce a consistent estimate of the optimum of
the function.
Our results show that such functions are negligible under the
distribution of the Gaussian process used to construct the sequential
design.
Further study of the convergence \emph{rate} of SUR sequential designs
is nevertheless needed to provide a full theoretical support for their
practical effectiveness, and will be the subject of future work.
Understanding their consistency, in relation with the properties of
the uncertainty functionals that defined them, is an important first
step in this direction.

The rest of the paper is structured as follows.
Section~\ref{sec:model-and-basic-results} defines more precisely the
statistical model and design problem considered in the paper, and
addresses properties of conditioning and convergence of Gaussian
measures that are instrumental in proving the main results of the
paper.
Section~\ref{sec:SUR} discusses uncertainty functionals and their
properties, and formulates general sufficient conditions for the
consistency of SUR sequential designs in terms of properties of the
associated uncertainty functionals.
Finally, Section~\ref{sec:examples} presents applications of the
general result to four popular examples of SUR sequential designs,
establishing in each case both convergence to zero for the considered
measure of residual uncertainty and convergence of the corresponding
estimator to the quantity of interest, in the almost sure and $L^1$
sense.

\section{Preliminaries: Gaussian processes and sequential design}
\label{sec:model-and-basic-results}

\subsection{Model}
\label{sec:bkgrnd:model}

Let $\left( \xi(x) \right)_{x \in \Xset}$ denote a Gaussian process
with mean function~$m$ and covariance function~$k$,
defined on a probability space~$\left( \Omega, \Fcal, \Prob \right)$
and indexed by a metric space~$\Xset$.
Assume that~$\xi$ can be observed at sequentially selected
(data-dependent) design points~$X_1$, $X_2$, \ldots, with additive
heteroscedastic Gaussian noise:
\begin{equation} \label{eq:iterative:sampling:definition}%
  Z_n = \xi(X_n) + \tau(X_n)\, U_n,%
  \qquad n = 1,\, 2,\, \ldots
\end{equation}
where $\tau:\Xset \to \Rplus$ gives the (known) standard
deviation~$\tau(x)$ of an observation at~$x \in \Xset$, and
$(U_i)_{i\geq 1}$ denotes a sequence of independent and identically
distributed~$\Ncal (0, 1)$ variables, independent of~$\xi$.
Let $\Fcal_n$ denote the $\sigma$-algebra generated by
$ X_1,Z _1, \ldots, X_n, Z_n$.

\begin{definition}
  A sequence~$\left( X_n \right)_{n \ge 1}$ will be said to form a
  (non-randomized) sequential design if, for all~$n \ge 1$, $X_n$ is
  $\Fcal_{n-1}$-measurable.
\end{definition}

\begin{stass} \label{ass:standing:assumption}
  We will assume in the rest of the paper that
  \begin{enumerate}[i) ]
  \item $\Xset$ is a compact metric space,
  \item $\xi$ has continuous sample paths,
  \item $\tau$ is continuous.
  \end{enumerate}
\end{stass}

\begin{remark}
  Note that the variance function~$\tau^2$ is not assumed to be
  strictly positive.
  Indeed, the special case where $\tau^2 \equiv 0$ is actually an
  important model to consider given its widespread use in Bayesian
  numerical analysis \citep[see, e.g.][]{Diaconis1988, o1991bayes,
    ritter00, hennig2015probnum} and in the design and analysis of
  deterministic computer experiments \cite[see, e.g.,][]{DACE, TDACE,
    bayarri2012framework}.
\end{remark}

\begin{remark} \label{rem:GP-continuity}
  A Gaussian process with continuous sample paths automatically has
  continuous mean and covariance functions (see, e.g., Lemma~1 in
  \cite{GRP}).
  Conversely, assuming continuity of the mean function, let us recall
  a classical sufficient condition for sample path continuity on
  $\Xset \subset \Rset^d$ \citep[see, e.g.,][Theorem~3.4.1]{Adler81}:
  if there exist $C > 0$ and $\eta > 0$ such that
  \begin{equation*}
    k(x,x) + k(y,y) - 2 k(x,y)
    \le \frac{C}{\left| \log \lVert x - y \rVert \right|^{1 + \eta}},
    \qquad
    \forall x,y \in \Xset,
  \end{equation*}
  then there exists a version of~$\xi$ with continuous sample paths.
\end{remark}

\begin{remark} \label{rem:Bayes-interp}
  The setting described in this section arises, notably, when
  considering from a Bayesian point of view the following
  nonparametric interpolation/regression model with heteroscedastic
  Gaussian noise:
  \begin{equation}
    Z_n = f(X_n) + \tau(X_n)\, U_n,\qquad
    n = 1,\, 2,\, \ldots
  \end{equation}
  with a continuous Gaussian process prior on the unknown regression
  function~$f$.
  In this case, $m$~and~$k$ are the \emph{prior} mean and covariance
  functions of~$\xi$.
\end{remark}

\subsection{Gaussian random elements
  and Gaussian measures on~$\CcalX$}
\label{sec:Gauss-RE-RM}

Let $\Sset = \CcalX$ denote the set of all continuous functions
on~$\Xset$.
Since $\Xset$ is assumed compact, $\Sset$~becomes a separable Banach
space when equipped with the supremum
norm~$\left\lVert \cdot \right\rVert_\infty$.
We recall \citep[see, e.g.,][Theorem~2.9]{azais09level} that any
Gaussian process~$\left( \xi(x) \right)_{x \in \Xset}$ with continuous
sample paths on a compact metric space satisfies
$\Esp( \left\lVert \xi \right\rVert_\infty) < \infty$.

Any Gaussian process~$\left( \xi(x) \right)_{x \in \Xset}$ with
continuous sample paths can be seen as a Gaussian random element
in~$\Sset$.
More precisely, the mapping
$\xi: \Omega \to \Sset,\, \omega \mapsto \xi(\omega, \cdot)$ is
$\Fcal / \Scal$-measurable, where $\Scal$ denotes the Borel
$\sigma$-algebra on~$\Sset$, and the probability
distribution~$\Prob^\xi$ of~$\xi$ is a Gaussian measure on~$\Sset$.
The reader is referred to \citet{vakhania87}
and~\citet{ledoux2013probability} for background information
concerning random elements and measures in Banach spaces, and to
\citet{van2008reproducing} and~\citet{bogachev1998} for more
information on the case of Gaussian random elements and measures.

We will denote by $\MGCX$ the set of all Gaussian measures on~$\Sset$.
Any $\nu \in \MGCX$ is the probability distribution of some Gaussian
process with continuous sample paths, seen as a random element
in~$\Sset$.
The mean function~$m_\nu$ and covariance function~$k_\nu$ of this
Gaussian process are continuous (see Remark~\ref{rem:GP-continuity})
and fully characterize the measure, which we will denote
as~$\GM (m_\nu, k_\nu)$.
We endow~$\MGCX$ with the $\sigma$-algebra $\McalGCX$ generated by the
evaluation maps~$\pi_A:\nu \mapsto \nu(A)$, $A \in \Scal$.
Using this $\sigma$-algebra, conditional distributions
on~$\Sset$---seen as transition kernels from~$\Omega$ to~$\Sset$---can
be conveniently identified to random elements in~$\MGCX$ \citep[see,
e.g.][p.~105--106]{kall02}.

Given a Gaussian random element~$\xi$ in~$\Sset$, we will denote by
$\Pfrak(\xi)$ the set of all \emph{Gaussian conditional distributions}
of~$\xi$, i.e., the set of all random Gaussian measures~$\nug$ such
that $\nug = \Prob\bigl( \xi \in \cdot \mid \Fcal' \bigr)$ for some
$\sigma$-algebra $\Fcal' \subset \Fcal$.
Note that we use a bold letter~$\nug$ to denote a random element
in~$\MGCX$ (i.e., a random Gaussian measure), and a normal
letter~$\nu$ to denote a point in the same space (i.e., a Gaussian
measure).
Not all conditional distributions of the
form~$\nug = \Prob\bigl( \xi \in \cdot \mid \Fcal' \bigr)$ are
Gaussian, but an important class of such Gaussian conditional
distributions is discussed in the following section and in
Proposition~\ref{prop:conv-M:n-to-infty}.

\subsection{Conditioning on finitely many observations}
\label{sec:cond-finitely-many}

It is well known that Gaussian processes remain Gaussian under
conditioning with respect to pointwise evaluations, or more generally
linear combinations of pointwise evaluations, possibly corrupted by
independent additive Gaussian noise
(explicit expressions of the conditional mean and covariance functions
are recalled in Appendix~\ref{app:conditioning}).
In the language of nonparametric Bayesian statistics (see
Remark~\ref{rem:Bayes-interp}), Gaussian process priors are
\emph{conjugate} with respect to this sampling model.
The following result formalizes this fact in the framework of Gaussian
measures on~$\Sset$, and states that the conjugation property still
holds when the observations are made according to a sequential design.

\begin{proposition} \label{prop:Cond-map}
  For all $n \ge 1$, there exists a measurable mapping
  \begin{equation}
    \label{eq:Cond-map}
    \begin{aligned}
      (\Xset \times \Rset)^n \times \MGCX %
      & \to \MGCX,\\
      (x_1, z_1, \ldots, x_n, z_n, \nu) %
      & \mapsto \Cond_{x_1,z_1,\ldots,x_n,z_n}(\nu),
    \end{aligned}
  \end{equation}
  such that, for any $\Prob^\xi \in \MGCX$ and any sequential
  design~$(X_n)_{n \ge 1}$,
  $\Cond_{X_1, Z_1, \ldots, X_n, Z_n}(\Prob^\xi)$ is a conditional
  distribution of~$\xi$ given~$\Fcal_n$.
\end{proposition}

A proof of this result is provided in Appendix~\ref{app:conditioning}.
In the rest of the paper, we will denote
by~$\Prob_n^\xi = \GM (m_n, k_n )$ the conditional distribution
$\Cond_{X_1, Z_1, \ldots, X_n, Z_n}(\Prob^\xi)$ of~$\xi$
given~$\Fcal_n$, which can be seen as a random element
in~$\left( \MGCX, \McalGCX \right)$.
The posterior mean~$m_n$ is also referred to as the \emph{kriging
  predictor} \citep[see, e.g.,][]{ISDSTK, forrester08, SSD}.
Note that~$m_n$ (respectively~$k_n$) is an $\Fcal_n$-measurable
process\footnote{i.e., a measurable process when considered as defined
  on~$\left( \Omega, \Fcal_n \right)$ instead
  of~$\left( \Omega, \Fcal \right)$} on~$\Xset$ (respectively
$\Xset \times \Xset$), with continuous sample paths.
Note also that $m_0 = m$ and~$k_0 = k$.
Conditionally to~$\Fcal_n$, the next observation follows a normal
distribution:
\begin{equation}
  Z_{n+1} \mid \Fcal_n \;\sim\;
  \Ncal \left( m_n(X_{n+1}),\, s_n^2(X_{n+1}) \right),
\end{equation}
where $s_n^2(x) = k_n(x, x) + \tau^2(x)$.

\subsection{Convergence in $\MGCX$}

We consider in this paper the following notion of convergence
on~$\MGCX$:

\begin{definition} \label{def:topology:Mzero}
  Let $\nu_n = \GM \bigl( m_n, k_n \bigr) \in \MGCX$,
  $n \in \Nset \cup \{+\infty\}$.
  We will say that~$\left( \nu_n \right)$ converges to~$\nu_\infty$,
  and write~$\nu_n \to \nu_\infty$, if $m_n \to m_\infty$ uniformly
  on~$\Xset$ (i.e., $m_n \to m_\infty$ in~$\Sset$) and
  $k_n \to k_\infty$ uniformly on~$\Xset \times \Xset$.
\end{definition}

\begin{remark}
  In other words, we consider the topology on~$\MGCX$ induced by the
  strong topology on the Banach
  space~$\CcalX \times \Ccal(\Xset\times\Xset)$, where $\MGCX$ is
  identified to a subset of this space through the injection
  $\nu \mapsto (m_\nu, k_\nu)$.
\end{remark}

Let us now state two important convergence results in this topology,
that will be needed in Section~\ref{sec:SUR}.
In the first of them, and later in the paper, we denote by
$\Fcal_\infty = \bigvee_{n \ge 1} \Fcal_n$ the $\sigma$-algebra
generated by~$\bigcup_{n \ge 1} \Fcal_n$.

\begin{proposition} \label{prop:conv-M:n-to-infty}
  For any Gaussian random element~$\xi$ in~$\Sset$, defined on any
  probability space, %
  and for any sequential design~$(X_n)_{n \ge 1}$, %
  the conditional distribution of~$\xi$ given $\Fcal_\infty$ admits a
  version~$\Prob_\infty^\xi$ which is an $\Fcal_\infty$-measurable
  random element in~$\MGCX$, and $\Prob_n^\xi \to \Prob_\infty^\xi$
  almost surely.
\end{proposition}

\begin{proposition} \label{prop:conv-M:cond-pt}
  Let $\nu \in \MGCX$ and let $(x_j, z_j) \to (x, z)$
  in~$\Xset \times \Rset$.
  Assume that $k_\nu(x, x) + \tau^2(x) > 0$.
  Then $\Cond_{x_j,\, z_j} (\nu) \to \Cond_{x,\, z} (\nu)$.
\end{proposition}

\begin{proof}
  See Appendix~\ref{app:ASCSCD:proofs} for proofs of both results.
\end{proof}

\section{Stepwise Uncertainty Reduction}
\label{sec:SUR}

\subsection{Uncertainty functionals and uncertainty reduction}
\label{sec:SUR:ufunc-and-ured}

As explained in the introduction, the definition of a SUR strategy
starts with the choice of an uncertainty functional~$\Hcal$, which
maps the conditional distribution~$\Pnxi$ to a measure~$H_n$ of
residual uncertainty for the quantity of interest.
The minimal value of the uncertainty functional represents an absence
of uncertainty on the quantity of interest:
for clarity, and without loss of generality as long as~$\Hcal$ is
bounded from below and attains its minimum, we will assume in the rest
of this section that $\min \Hcal = 0$, thus restricting our attention
to non-negative uncertainty functionals.

More formally, let~$\Hcal$ denote a measurable functional from~$\MGCX$
to~$\left[ 0, +\infty \right)$.
Since $\Pnxi$ is an $\Fcal_n$-measurable random element
in~$\left( \MGCX, \McalGCX \right)$, the residual uncertainty
$H_n = \Hcal (\Pnxi)$ is an $\Fcal_n$-measurable random variable.
A key observation for the convergence results of this paper is that
many uncertainty functionals of interest---examples of which will be
given in Section~\ref{sec:examples}---enjoy the following property:
\begin{definition}
  A measurable functional~$\Hcal$ on~$\MGCX$ will be said to have the
  \emph{supermartingale property} if, %
  for any Gaussian random element~$\xi$ in~$\Sset$, %
  defined on any probability space, %
  and for any sequential design~$(X_n)_{n \ge 1}$, %
  the sequence~$\left(\Hcal (\Pnxi)\right)_{n \ge 0}$ is an
  $\left(\Fcal_n\right)$-supermartingale.
\end{definition}

The supermartingale property echoes DeGroot's observation that
``reasonable'' measures of uncertainty should be decreasing on average
for any possible experiment \citep{degroot1962}.
To discuss this connection more precisely in our particular setting,
let us consider the following definition.
\begin{definition} \label{def:DoA}
  Let $\MsetQcq$ denote a set of probability measures on a measurable
  space~$\left( \Eset, \Ecal \right)$, and let~$\McalQcq$
  denote the $\sigma$-algebra generated on~$\MsetQcq$
  by the evaluation maps.
  For any random element~$\nug$
  in~$\left( \MsetQcq, \McalQcq \right)$, defined on any probability
  space, let $\onug$ denote the probability measure defined by
  $\onug (A) = \Esp \left( \nug(A) \right)$, $A \in \Ecal$.
  We will say that a non-negative measurable functional~$\Hcal$
  on~$\MsetQcq$ is \emph{decreasing on average} (DoA) if, for any
  random element~$\nug$ in~$\left( \MsetQcq, \McalQcq \right)$ such
  that $\onug \in \MsetQcq$,
  $\Esp\left( \Hcal(\nug) \right) \le \Hcal(\onug)$.
\end{definition}

Note that, if the set~$\MsetQcq$ is convex, DoA functionals
on~$\MsetQcq$ are concave.
The converse statement is expected to be false, however, since
Jensen's inequality does not hold for all concave functionals in
infinite dimensional settings \citep[see][for extensions of Jensen's
inequality under various assumptions]{perlman74}.
The set~$\MGCX$ of all Gaussian measures on~$\Sset$ is not convex, but
all the uncertainty functionals presented in
Section~\ref{sec:examples} can in fact be extended to DoA---hence
concave---functionals defined on some larger convex set of probability
measures\footnote{%
  More precisely, the functionals discussed in
  Sections~\ref{sec:example:pUnMoinsp} and~\ref{sec:example:variance}
  can be extended to DoA functionals on the set of all probability
  measures on~$\Sset$,
  and those of Sections~\ref{sec:example:knowledge}
  and~\ref{sec:example:ei} to DoA functionals on the set of all
  probability measures~$\nu$ on~$\Sset$ such that
  $\Esp\left( \max \xi - \min \xi \right) < +\infty$
  for~$\xi \sim \nu$.
}.

\begin{remark}
  Let $\beta: \Sset \to \Rset^p$ denote a measurable function,
  and let $\nu^\beta$ denote the image of~$\nu$ by~$\beta$.
  Then it is easy to see that any functional of the form
  $\Hcal(\nu) = \Hcal' ( \nu^\beta )$ is DoA, where $\Hcal'$ denotes
  a DoA functional defined on some appropriate subset of the set of
  all probability measures on~$\Rset^p$;
  the reader is referred to \cite{hainy2014} for a variety of examples
  of such functionals.
  Section~\ref{sec:example:variance} provides an example of this
  construction, with $p = 1$ and $\Hcal'$ the variance functional.
\end{remark}

The supermartingale and DoA properties are easily seen to be connected
as follows:
\begin{proposition}
  If $\Hcal$ is DoA on $\MGCX$, then it has the supermartingale
  property.
\end{proposition}

Let us conclude this section with another useful property of
functionals.
Recall that $\Pfrak(\xi)$ denotes the set of all Gaussian conditional
distributions of~$\xi$ (see Section~\ref{sec:Gauss-RE-RM}).

\begin{definition}
  A measurable functional~$\Hcal$ on~$\MGCX$ will be said to be
  \emph{\UiFgcd} if, for any Gaussian random element~$\xi$
  in~$\Sset$, %
  defined on any probability space, %
  the family $\left( \Hcal(\nug) \right)_{\nug \in \Pfrak(\xi)}$ is
  uniformly integrable.
\end{definition}

\begin{proposition} \label{prop:UiFgcd:upperbound}
  Let $\Hcal$ denote a measurable functional on~$\MGCX$.
  If there exists
  $L^+ \in \cap_{\nu \in \MGCX} \mathcal{L}^1 \left( \Sset, \Scal, \nu
  \right)$ such that
  $\left| \Hcal (\nu) \right| \le \int_\Sset L^+ \dnu$ for all
  $\nu \in \MGCX$, then $\Hcal$ is \UiFgcd.
\end{proposition}

\begin{proof}
  Let $\xi$ denote a Gaussian random element in~$\Sset$
  and let
  $\nug = \Prob\bigl( \xi \in \cdot \mid \Fcal' \bigr) \in
  \Pfrak(\xi)$.
  Then we have
  $\vert \Hcal(\nug) \vert \le \Esp \left( L^+(\xi) \mid \Fcal'
  \right)$,
  and the result follows from the uniform integrability of conditional
  expectations \citep[see, e.g.,][Lemma~5.5]{kall02}.
\end{proof}

\begin{remark}
  If $\Hcal$ is \UiFgcd and has the supermartingale property then, for
  any sequential design,
  the sequence~$(H_n)$ is a uniformly integrable supermartingale
  (since $\{ \Prob_n^\xi \} \subset \Pfrak(\xi)$), and thus converges
  almost surely and in~$L^1$.
\end{remark}

\subsection{SUR sequential designs and associated functionals}
\label{sec:SUR:def}

The SUR sampling criterion introduced informally as
$J_n(x) = \Esp_{n,x} \left( H_{n+1} \right)$
in~\eqref{eq:generic-SUR:sampcrit-Jn} can now be more precisely
defined as
\begin{equation}
  J_n(x) %
  = \Esp_n\left( \Hcal\left( \Cond_{x,\, Z_{n+1}(x)}
      \left( \Prob_n^\xi \right) \right) \right),
  \label{equ:Jn-def-rigoureuse}
\end{equation}
where $Z_{n+1}(x) = \xi(x) + U_{n+1}\, \tau(x)$.
A SUR sequential design is then built by selecting at each step,
possibly after some initial design, the next design point as a
minimizer of the SUR sampling criterion~$J_n$:
\begin{definition}
  \label{def:SUR-qSUR}
  Let $\Hcal$ denote a non-negative measurable functional on~$\MGCX$.
  \begin{enumerate}[i) ]
  \item \label{def:SUR-qSUR:i}%
    We will say that $\left( X_n \right)$ is a \emph{SUR sequential
      design} associated with the uncertainty functional~$\Hcal$ if it
    is a sequential design such that $X_{n+1} \in \argmin J_n$ for
    all~$n \ge n_0$, for some integer~$n_0$.
  \item \label{def:SUR-qSUR:ii}%
    Given a sequence $\varepsilon = \left( \varepsilon_n \right)$ of
    non-negative real numbers such that $\varepsilon_n \to 0$, we will say
    that $\left( X_n \right)$ is an \emph{$\varepsilon$-quasi-SUR}
    sequential design if it is a sequential design such that
    $J_n \left( X_{n+1} \right) \le \inf J_n + \varepsilon_n$ for
    all~$n \ge n_0$, for some integer~$n_0$.
  \end{enumerate}
\end{definition}

\begin{remark}
  In practice it is not always easy to guarantee that, for a given
  uncertainty functional~$\Hcal$, the sampling criteria~$J_n$ attain
  their infimum over~$\Xset$.
  Moreover, the actual minimization of~$J_n$ is typically carried out by
  means of a numerical optimization algorithm, which cannot be
  expected to provide the exact minimizer.
  For these reasons, it seems important to study the convergence
  of quasi-SUR designs, as introduced by
  Definition~\ref{def:SUR-qSUR}.\ref{def:SUR-qSUR:ii}, instead of the
  more restrictive case of (exact) SUR designs.
  General existence results for SUR and quasi-SUR designs, based on
  the measurable selection theorem for \RCSs, are
  provided in Appendix~\ref{sec:SUR:existence}.
\end{remark}

Let us now introduce some useful functionals associated to a
given (non-negative) uncertainty functional~$\Hcal$.
First, observe that $J_n(x) = \Jcal_x (\Pnxi)$, where the functional
$\Jcal_x: \MGCX \to \RplusExt$ is defined for all~$x \in \Xset$ and
$\nu \in \MGCX$ by
\begin{align}
  \Jcal_x(\nu)
  & = \iint_{\Sset \times \Rset} \Hcal
    \left( \Cond_{x,\, f(x) + u\, \tau(x)}\, (\nu)
    \right)\, \nu(\df)\, \phi(u)\, \du
    \label{equ:Jn-1ere-ecriture-intdouble}
  \\
  & = \int_{\Rset} \Hcal
    \left( \Cond_{x,\, m_\nu(x) + v\, s_\nu(x)} (\nu) \right)\, \phi(v)\, \dv,
    \label{equ:Jn-2eme-ecriture-intsimple}
\end{align}
with $s_\nu^2(x) = k_\nu(x,x) + \tau^2(x)$ and $\phi$ the probability
density function of the standard normal distribution.
The mapping $(x, \nu) \mapsto \Jcal_x(\nu)$ is
$\Bcal(\Xset) \otimes \McalGCX$-measurable (see
Proposition~\ref{prop:Jcalx:meas}),
and it is easy to see that $\Hcal$ has the supermartingale
property if, and only if,
\begin{equation}
  \label{equ:Jn-smaller-than-Hn}
  \Jcal_x(\nu) \le \Hcal(\nu), \quad %
  \text{for all } x \in \Xset \text{ and } \nu \in \MGCX.
\end{equation}

Assuming that $\Hcal$ has the supermartingale property, we will then
denote by $\Gcal_x: \MGCX \to \Rplus$ the corresponding
\emph{expected gain functional} at~$x$:
\begin{equation}
  \Gcal_x (\nu) = \Hcal(\nu) - \Jcal_x(\nu), \label{eq:def-Gx}
\end{equation}
and by $\Gcal: \MGCX \to \Rplus$ the associated \emph{maximal expected
  gain functional}:
\begin{equation}
  \Gcal (\nu) = \sup_{x\in\Xset}\, \Gcal_x (\nu). \label{eq:def-G}
\end{equation}

\begin{remark}
  Following \cite{degroot1962}, $\Gcal_x$ could be called the
  ``information'' brought by an evaluation at~$x$ about the quantity
  of interest.
  This would be consistent with the usual definition of mutual
  information, when~$\Hcal$ is taken to be the Shannon entropy of some
  discrete quantity of interest %
  \citep[see, e.g.,][]{coverthomas}.
  Note that DeGroot renamed it ``expected information'' in some of his
  subsequent work on this topic \citep[see, e.g.,][]{DeGroot1986,
    DeGroot1994}.
\end{remark}

\begin{remark}
    Alternatively, SUR sequential designs can be defined by the
    relation $X_{n+1} \in \argmax G_n$, where $G_n$ denotes the
    sampling criterion
    $x \mapsto G_n(x) \eqdef \Gcal_x(\Pnxi) = H_n - J_n(x)$.
    In the particular cases discussed in
    Sections~\ref{sec:example:knowledge} and~\ref{sec:example:ei},
    $G_n$ corresponds to the \emph{knowledge gradient} and
    \emph{expected improvement} criteria, respectively.
\end{remark}

\subsection{General convergence results}
\label{sec:SUR:general-conv}

Denote by~$\Zset_\Hcal$ and~$\Zset_\Gcal$ the subsets of~$\MGCX$ where
the functionals~$\Hcal$ and~$\Gcal$ vanish, respectively.
The inclusion $\Zset_\Hcal \subset \Zset_\Gcal$ always hold: indeed,
$0 \le \Jcal_x \le \Hcal$ for all~$x$
by~\eqref{equ:Jn-smaller-than-Hn}, thus $0 \le \Gcal_x \le \Hcal$, and
therefore $0 \le \Gcal \le \Hcal$.
The reverse inclusion plays a capital role in the following result,
which provides sufficient conditions for the almost sure convergence
of quasi-SUR sequential designs associated with uncertainty
functionals that enjoy the supermartingale property.

\begin{theorem} \label{thm:convergence:generic} %
  Let $\Hcal$ denote a non-negative, measurable functional on~$\MGCX$
  with the supermartingale property.
  Let $\left( X_n \right)$ denote a quasi-SUR sequential design
  for~$\Hcal$.
  Then $\Gcal\bigl( \Prob_n^\xi \bigr) \to 0$ almost surely.
  If, moreover,
  \begin{enumerate}[i) ] \setcounter{enumi}{\value{enumiSave}}
  \item \label{thm:convergence:generic:HnCont}
    $H_n = \Hcal \bigl(\Prob_n^\xi\bigr) \to \Hcal \bigl(
    \Prob_\infty^\xi \bigr)$ almost surely,
  \item \label{thm:convergence:generic:GnCont}
    $\Gcal \bigl( \Prob_n^\xi \bigr) \to \Gcal \bigl( \Prob_\infty^\xi
    \bigr)$ almost surely,
  \item \label{thm:convergence:generic:ZZ}
    $\Zset_\Hcal = \Zset_\Gcal$,
  \end{enumerate}
  then $H_n \to 0$ almost surely.
\end{theorem}

The proof of Theorem \ref{thm:convergence:generic} relies on two main
ideas.
First, because the sequence~$\left( H_n \right)$ is a non-negative
supermartingale, the conditional mean of its increments goes to zero
almost surely, which implies by the quasi-SUR assumption that the
maximal expected gain goes to zero as well.
Second, using Assumptions~\ref{thm:convergence:generic:HnCont}
and~\ref{thm:convergence:generic:GnCont}, it is enough to study the
limiting distribution~$\Prob_\infty^\xi$: this is where the reverse
inclusion $\Zset_\Gcal \subset \Zset_\Hcal$ is used to conclude that
the uncertainty in the limiting distribution is zero.

\begin{proof}
  %
  %
  Since $X_{n+1}$ is $\Fcal_n$-measurable, we have:
  \begin{equation}
    \label{eq:Jn-Xnp1-gotcha}
    \begin{aligned}
      J_n \left( X_{n+1} \right) %
      & = \Esp_n\left( %
      \Hcal \left(
        \Cond_{x, Z_{n+1}(x)} \bigl( \Pnxi \bigr)
         \right)
         \right)_{|x = X_{n+1}}\\
      & = \Esp_n\left(
      \Hcal \left(
      \Cond_{X_{n+1}, Z_{n+1}} \bigl( \Pnxi \bigr)
     \right)
      \right) %
      = \Esp_n \left( H_{n+1} \right).
    \end{aligned}
  \end{equation}
  Set $\Delta_{n+1} = H_n - H_{n+1}$ and
  $\overline \Delta_{n+1} = \Esp_n \left( \Delta_{n+1} \right) = H_n -
  \Esp_n \left( H_{n+1} \right)$.
  The random variables~$\overline \Delta_n$ are non-negative since
  $\left( H_n \right)$ is a supermartingale and, using
  that~$\left( X_n \right)$ is an $\varepsilon$-quasi-SUR design, we
  have for all~$n \ge n_0$:
  \begin{equation}
    \overline \Delta_{n+1} %
    = H_n - \Esp_n \left( H_{n+1} \right)
    = H_n - J_n \left( X_{n+1} \right)
    \ge H_n - \inf_{x \in \Xset} J_n(x) - \varepsilon_n,
  \end{equation}
  i.e., since $J_n(x) = \Jcal_x \bigl( \Pnxi \bigr)$ and
  $\Gcal_x = \Hcal - \Jcal_x$,
  \begin{equation}
    \overline \Delta_{n+1} %
    \ge \sup_{x \in \Xset} \Gcal_x \bigl( \Pnxi \bigr) - \varepsilon_n
    = \Gcal \bigl( \Pnxi \bigr) - \varepsilon_n.
  \end{equation} 
  Moreover, for any $n$, we have
  $\sum_{k=0}^{n-1} \Delta_k = H_0 - H_n$, and therefore
  \begin{equation*}
    \Esp \left( \sum_{k=0}^{n-1} \overline \Delta_k \right)
    = \Esp \left( \sum_{k=0}^{n-1} \Delta_k \right)
    = \Esp \left( H_0 - H_n \right)
    \le \Esp \left( H_0 \right) < +\infty.
  \end{equation*}
  It follows that
  $\Esp \left( \sum_{k=0}^\infty \overline \Delta_k \right) <
  +\infty$, and thus $\overline \Delta_n \to 0$ almost surely.
  As a consequence, $\Gcal \bigl( \Pnxi \bigr) \to 0$ almost surely,
  since
  $0 \le \Gcal \bigl( \Pnxi \bigr) \le \overline \Delta_{n+1} +
  \varepsilon_n$.

  %
  %
  Let now
  Assumptions~\ref{thm:convergence:generic:HnCont}--\ref{thm:convergence:generic:ZZ}
  hold.  It follows from the first part of the proof that
  $\Gcal \bigl( \Prob_n^\xi \bigr) \to 0$ almost surely.
  Thus, $\Gcal \bigl( \Prob_\infty^\xi \bigr) = 0$ almost surely
  according to Assumption~\ref{thm:convergence:generic:GnCont}.  Then
  $\Hcal \bigl( \Prob_\infty^\xi \bigr) = 0$ since
  $\Zset_\Gcal \subset \Zset_\Hcal$, and the conclusion follows from
  Assumption~\ref{thm:convergence:generic:HnCont}.
\end{proof}

\begin{remark}
  Note that the conclusions of Theorem~\ref{thm:convergence:generic}
  still hold partially if it is only assumed that the condition
  $J_n \left( X_{n+1} \right) \le \inf J_n + \varepsilon_n$ holds
  infinitely often, almost surely:
  in this case the conclusion of the first part of the theorem is
  weakened to $\liminf \Gcal\bigl( \Prob_n^\xi \bigr) = 0$, but the
  final conclusion ($H_n \to 0$ a.s.) remains the same.
\end{remark}

Since $\Prob_n^\xi \to \Prob_\infty^\xi$ almost surely by
Proposition~\ref{prop:conv-M:n-to-infty},
Assumptions~\ref{thm:convergence:generic:HnCont}
and~\ref{thm:convergence:generic:GnCont} of
Theorem~\ref{thm:convergence:generic} hold if $\Hcal$ and~$\Gcal$,
respectively, are continuous.
Assuming~$\Hcal$ to be continuous, however, would be too strong a
requirement, that some important examples would fail to satisfy.
For instance, the uncertainty functional
\begin{equation}
  \label{equ:H-punmoinsp:preview}
  \Hcal: \nu \mapsto \int_\Xset p_\nu (1 - p_\nu)\, \dmu
\end{equation}
studied in Section~\ref{sec:example:pUnMoinsp}, %
where $p_\nu(u) = \int_\Sset \one_{f(u) \geq T}\, \nu(\df)$ for some
threshold $T \in \Rset$, is clearly discontinuous at the degenerate
measure~$\nu = \GM \left( T \one_\Xset,\, 0 \right)$.
The following weaker notion of continuity will turn out to be suitable
for our needs:
\begin{definition} \label{def:AscFgcd}
  A measurable functional $\Hcal$ on~$\MGCX$ will be said to be
  \emph{\AscFgcd} if, %
  for any Gaussian random element~$\xi$ in~$\Sset$, %
  defined on any probability space, %
  and any sequence of random measures~$\nug_n \in \Pfrak(\xi)$ such
  that $\nug_n \toas \nug_{\infty} \in \Pfrak(\xi)$, %
  the convergence $\Hcal(\nug_n) \toas \Hcal(\nug_\infty)$ holds.
\end{definition}

\begin{remark}
  The uncertainty functional~\eqref{equ:H-punmoinsp:preview} provides
  an explicit example of a functional which is \AscFgcd (cf. the proof
  of Theorem~\ref{thm:p:un:moins:p}) but not continuous.
  The expected improvement functional, discussed in
  Section~\ref{sec:example:ei}, provides an example of a functional
  which is not even \AscFgcd (see
  Proposition~\ref{prop:EI:not:regular}), %
  but for which consistency can nonetheless be proved by a direct
  application of Theorem~\ref{thm:convergence:generic}.
\end{remark}

Checking that~$\Gcal$ is \AscFgcd, however, is not easy in practice.
The following results provides sufficient conditions for
Assumption~\ref{thm:convergence:generic}.\ref{thm:convergence:generic:GnCont}
that are easier to check.

\begin{theorem} \label{thm:convergence:bis}
  Let $\Hcal$ denote a non-negative, measurable uncertainty functional
  on~$\MGCX$, and let~$\Gcal$ denote the associated maximal expected
  gain functional.
  Assume that $\Hcal = \Hcal_0 + \Hcal_1$, where
  \begin{enumerate}[i)]
  \item $\Hcal_0(\nu) = \int_\Sset L_0\, \dnu$ for some
    $L_0 \in \cap_{\nu \in \MGCX} \mathcal{L}^1 \left( \Sset, \Scal,
      \nu \right)$, and
  \item $\Hcal_1$ is \UiFgcd, \AscFgcd and has the supermartingale
    property.
  \end{enumerate}
  Then, for any quasi-SUR sequential design associated with~$\Hcal$,
  $\Gcal \bigl( \Prob_\infty^\xi \bigr) = 0$ almost surely.
\end{theorem}

\begin{proof}
  First, note that
  $\Hcal_0\PxiParen{n} = \Esp_n \left( L_0(\xi) \right)$.
  Thus, since
  $L_0 \in \mathcal{L}^1 \left( \Sset, \Scal, \Prob^\xi \right)$, the
  sequence $\left( \Hcal_0(\Pnxi) \right)$ is a uniformly integrable
  martingale (see, e.g., Kallenberg~\cite{kall02}, Theorem~6.23),
  which converges almost surely and in~$L^1$ to
  $\Esp_\infty \left( L_0(\xi) \right) = \Hcal_0\PxiParen{\infty}$.
  As a consequence, $\Hcal\PxiParen{n} \toas \Hcal\PxiParen{\infty}$
  since $\Hcal_1$ is \AscFgcd and $\Pnxi \toas \Pinfxi$ by
  Proposition~\ref{prop:conv-M:n-to-infty}.

  Let $x \in \Xset$.
  The functional $\Hcal_0$ has the supermartingale property by the
  preceding argument, and therefore $\Hcal = \Hcal_0 + \Hcal_1$ also
  has the supermartingale property.
  Then, it follows from the first part of
  Theorem~\ref{thm:convergence:generic} that
  $\Gcal_x \PxiParen{n} \toas 0$,
  and thus
  \begin{equation}
    \label{eq:Jn:limit1}
    \Jcal_x\PxiParen{n}
    = \Hcal\PxiParen{n} - \Gcal_x\PxiParen{n}
    \toas \Hcal\PxiParen{\infty}. %
  \end{equation}
  Let $\Prob_{n,x}^\xi = \Cond_{x, Z(x)} \left( \Pnxi \right)$, with
  $Z(x) = \xi(x) + \tau(x)\, U$ and $U \sim \Ncal(0, 1)$ independent
  from~$\xi$ and the~$U_n$'s, and observe
  that~$\Jcal_x\PxiParen{n} = \Esp_n\left( \Hcal\PxiParen{n,x}
  \right)$.
  Consider then the decomposition:
  \begin{align}
    \Jcal_x\PxiParen{n}
    & = \Esp_n\left( \Hcal\PxiParen{n,x} - \Hcal\PxiParen{\infty,x} \right)
      + \Esp_n\left( \Hcal\PxiParen{\infty,x} \right)
    \nonumber\\
    & = \Esp_n\left( \Hcal_1\PxiParen{n,x} - \Hcal_1\PxiParen{\infty,x} \right)
      + \Esp_n\left( \Hcal\PxiParen{\infty,x} \right)
    \label{equ:Jn-decomp}
  \end{align}
  where the second equality simply follows from the fact that
  $\Esp_n\left( \Hcal_0\PxiParen{n,x} \right) = \Esp_n\left(
    \Hcal_0\PxiParen{\infty,x} \right) = \Esp_n\left( L_0(\xi)
  \right)$ by the law of total expectation.
  The second conditional expectation in~\eqref{equ:Jn-decomp} is,
  again, a uniformly integrable martingale that converges almost
  surely and in~$L^1$:
  \begin{equation}
    \label{equ:Jn-decomp:asympt1}
    \Esp_n\left( \Hcal\PxiParen{\infty,x} \right)
    \xrightarrow[n \to \infty]{\as,\, L^1}
    \Esp_{\infty}\left( \Hcal\PxiParen{\infty,x} \right).
  \end{equation}
  Moreover, note that
  \begin{align*}
    \Prob_{n,x}^\xi %
    & = \Cond_{X_1,\, Z_1,\, \ldots,\, X_n,\, Z_n,\, x,\, Z(x)}\, \PxiParen{0} \\
    & = \Cond_{x,\, Z(x),\, X_1,\, Z_1,\, \ldots,\, X_n,\, Z_n}\, \PxiParen{0}
  \end{align*}
  is the conditional distribution of~$\xi$ at the $(n+1)^{\text{th}}$
  step of the modified sequential design
  $\bigl( \widetilde X_n \bigr)$, where $\widetilde X_1 = x$ and
  $\widetilde X_{n+1} = X_n$ for all $n \ge 1$,
  with a modified sequence of ``noise variables''
  $\bigl( \widetilde U_n \bigr)$ defined by $\widetilde U_1 = U$ and
  $\widetilde U_{n+1} = U_n$ for all~$n \ge 1$.
  Note also that~$\Prob_{\infty,x}^\xi$ corresponds to the conditional
  distribution with respect to the~$\sigma$-algebra generated
  by~$\widetilde X_1, \widetilde Z_1, \widetilde X_2, \widetilde
  Z_2\ldots$, where the $\widetilde Z_n$'s have been defined
  accordingly.
  As a result,
  \begin{equation}
    \label{equ:Jn-decomp:asympt2}
    \Esp_n\left( \Hcal_1\PxiParen{n,x} - \Hcal_1\PxiParen{\infty,x} \right)
    \xrightarrow[n \to \infty]{L^1} 0
  \end{equation}
  since $\Hcal_1$ is \AscFgcd and \UiFgcd.
  Combine~\eqref{equ:Jn-decomp},
  \eqref{equ:Jn-decomp:asympt1} and~\eqref{equ:Jn-decomp:asympt2} to
  prove that
  $\Jcal_x\PxiParen{n} \to \Esp_{\infty} \bigl(
  \Hcal\PxiParen{\infty,x} \bigr)$ in~$L^1$.  Then, it follows from a
  comparison with~\eqref{eq:Jn:limit1} that
  $\Hcal\PxiParen\infty = \Esp_{\infty}\bigl( \Hcal\PxiParen{\infty,x}
  \bigr)$ almost surely, and therefore
  \begin{equation}
    \label{eq:Gcalx-zero}
    \Gcal_x\PxiParen\infty = \Hcal\PxiParen\infty
    - \Esp_{\infty}\left( \Hcal\PxiParen{\infty,x} \right)
    = 0 \quad \text{almost surely.}
  \end{equation}

  To conclude, note that by Assertion~\eqref{assert:Jn-continuous}
  of Theorem~\ref{thm:Jn-cont:and:SUR-exist} the sample paths of
  $J_\infty:x \mapsto \Esp_{\infty}\bigl( \Hcal\PxiParen{\infty,x}
  \bigr)$ are continuous
  on~$\left\{ x\in\Xset: s_\infty^2(x) > 0 \right\}$. %
  Let $\left\{ x_j \right\}$ denote a countable dense subset
  of~$\Xset$. %
  We have proved that, almost surely, $\Gcal_{x_j}\PxiParen\infty = 0$
  for all~$j$. %
  Using the continuity of~$J_\infty$
  on~$\left\{ s_\infty^2 > 0 \right\}$, and the fact
  that~$\Gcal_x = 0$ on~$\left\{ s_\infty^2 = 0 \right\}$,
  we conclude that, almost surely, $\Gcal_{x}\PxiParen\infty = 0$ for
  all~$x$, and therefore $\Gcal\PxiParen\infty = 0$, which concludes
  the proof.
\end{proof}

\subsection{Uncertainty functionals based on a loss function}
\label{sec:SUR:ufunc-loss}

Let us now consider, more specifically, %
uncertainty functionals~$\Hcal$ %
defined in the form of a risk:
\begin{equation}
  \label{eq:def-H-from-loss}
  \Hcal(\nu) = \inf_{d \in \Dset}\; \int_\Sset L(f,d)\, \nu(\df)
  = \inf_{d \in \Dset}\; {\overline L}_{\nu}(d),
\end{equation}
where~$\Dset$ is a set of ``decisions'',
$L:\Sset \times \Dset \to \RplusExt$ a ``loss function''
such that $L(\cdot,d)$ is $\Scal$-measurable for all~$d \in \Dset$, and
${\overline L}_{\nu}(d) = \int_\Sset L(f,d)\, \nu(\df)$. %
All the examples that will be discussed in Section~\ref{sec:examples}
can be written in this particular form.

The following result formalizes an important observation of
\citet[][p.~408]{degroot1962} about such uncertainty
functionals---namely, that they always enjoy the DoA property
introduced in Section~\ref{sec:SUR:ufunc-and-ured} %
(and thus can be studied using Theorem~\ref{thm:convergence:generic}).

\begin{proposition}
  \label{prop:H-infi-DoA}
  Let $\Hcal$ denote a measurable functional on~$\MGCX$.
  If $\Hcal$ is of the form~\eqref{eq:def-H-from-loss}, then it is
  DoA on~$\MGCX$, and consequently has the supermartingale property.
\end{proposition}

\begin{proof}
  The result follows directly from the fact that~$\Hcal$ is the
  infimum of a family of linear functionals
  ($\nu \mapsto {\overline L}_{\nu}(d)$, for $d \in \Dset$) %
  that commute with expectations: %
  for any random element~$\nug$ in~$\MGCX$ and any~$d \in \Dset$,
  \begin{equation}
   \label{equ:lin-fun-DoA}
    \Esp \left(\, {\overline L}_{\nug}(d) \right)
    = \Esp \left( \int_\Sset L(f,d)\, \nug(\df) \right)
    = {\overline L}_{\overline \nug}(d),
  \end{equation}
  where $\overline \nug$ is defined as in Definition~\ref{def:DoA}.
  (In other words, the linear functionals $\nu \mapsto L_\nu(d)$ are
  DoA themselves, with an equality in~\eqref{equ:lin-fun-DoA} instead
  of the inequality in Definition~\ref{def:DoA}.)
\end{proof}

An uncertainty functional of the form~\eqref{eq:def-H-from-loss} is
clearly $\Mcal$-measurable if the infimum over~$d$ can be restricted
to a countable subset of~$\Dset$ %
(since the linear functionals $\nu \mapsto {\overline L}_{\nu}(d)$ are
$\Mcal$-measurable by Lemma~\ref{lem:P-measurability-2}).
This is true, for instance, if $\Dset$ is separable and
$d \mapsto {\overline L}_{\nu}(d)$ is continuous for all~$\nu$.
See Proposition~\ref{prop:EI:H:measurable} for an example where~$L$
is discontinuous.

Three of the examples of SUR sequential designs from the literature
that will be analyzed in Section~\ref{sec:examples} are based on
regular non-negative loss functions in the following sense.

\begin{definition} \label{def:RegLoss}
  We will say that a non-negative loss function
  $L: \Sset \times \Dset \to \Rplus$ is \emph{regular}
  if
  \begin{enumerate}[i) ]
    %
    %
  \item $\Dset$ is a separable space,
    %
    %
  \item for all~$d \in \Dset$, $L(\cdot,d)$ is $\Scal$-measurable,
    %
    %
  \item for all $\nu \in \MGCX$, ${\overline L}_\nu$ takes finite
    values and is continuous on~$\Dset$,
    \setcounter{enumiSave}{\value{enumi}}
  \end{enumerate}
  and if the corresponding functionals~$\Hcal$ and~$\Gcal$ satisfy:
  \begin{enumerate}[i) ]
    \setcounter{enumi}{\value{enumiSave}}
    %
    %
  \item $\Hcal = \Hcal_0 + \Hcal_1$, where
    $\Hcal_0(\nu) = \int_\Sset L_0\, \dnu$ for some
    $L_0 \in \cap_{\nu \in \MGCX} \mathcal{L}^1 \left( \Sset, \Scal,
      \nu \right)$, and $\Hcal_1$ is \UiFgcd and \AscFgcd,
    %
    %
  \item $\Zset_\Hcal = \Zset_\Gcal$.
  \end{enumerate}
\end{definition}

The following corollary is provided as a convenient summary of the
results that hold for uncertainty functionals based on
regular non-negative loss functions.

\begin{corollary}
  \label{cor:RegLoss}
  Let $\Hcal$ denote a functional of the
  form~\eqref{eq:def-H-from-loss} for some non-negative loss
  function~$L$.
  If $L$ is regular, then
  $\Hcal$ is a measurable functional that
    satisfies the assumptions of
    Theorems~\ref{thm:Jn-cont:and:SUR-exist},
    \ref{thm:convergence:generic} and~\ref{thm:convergence:bis}.
  In particular,
  for any quasi-SUR design associated with~$\Hcal$, %
  $H_n = \Hcal \PxiParen{n} \to 0$ almost surely.
\end{corollary}

\section{Applications to popular sequential design strategies}
\label{sec:examples}

This section presents applications of our results to
four popular sequential design strategies, two of them addressing the
excursion case (Sections~\ref{sec:example:pUnMoinsp}
and~\ref{sec:example:variance}), and the other two addressing the
optimization case (Sections~\ref{sec:example:knowledge}
and~\ref{sec:example:ei}).
For each example, the convergence results are preceded by details on
the associated loss functions, uncertainty functionals and sampling
criteria.

\subsection{The integrated Bernoulli variance functional}
\label{sec:example:pUnMoinsp}

Assume that $\Xset$ is endowed with a finite
measure~$\mu$ and let $T \in \Rset$ be a given excursion
threshold.
For any measurable function $f:\Xset \mapsto \Rset$, let
$\Gamma(f)=\{u\in \Xset: f(u) \geq T\}$ and
$\alpha(f)=\mu(\Gamma(f))$.  The quantities of interest are then
$\Gamma(\xi)$ and $\alpha(\xi)$.
Let $p_n(u) = \Esp_n \left( \one_{\Gamma(\xi)} (u) \right) %
= \Prob_{n} \left( \xi(u) \geq T \right)$.
A typical choice of measure of residual uncertainty in this case is
the integrated indicator---or ``Bernoulli''---variance \citep{bect12}:
\begin{equation}
  \label{eq:uncertainty_bernoulli}
  H_n = \int_{\Xset} p_{n} \left( 1 - p_{n} \right) \dmu,
\end{equation}
which corresponds to the uncertainty functional
\begin{equation}
  \label{equ:ibv:Hcal}
  \Hcal(\nu) = \int_{\Xset} p_\nu \left(1 - p_\nu\right) \dmu,
  \qquad \nu \in \MGCX,
\end{equation}
where $p_\nu(u) = \int_\Sset \one_{f(u) \geq T}\, \nu(\df)$.
See \cite{chevalier14} for more information on the computation of the
corresponding SUR sampling criterion
\begin{equation*}
  J_n(x) = \Esp_{n, x} \left(
    \int_{\Xset} p_{n+1} \left( 1 - p_{n+1} \right) \dmu
  \right).
\end{equation*}

The functional~\eqref{equ:ibv:Hcal} can be seen as the uncertainty
functional induced by the loss function
\begin{equation}
  \label{equ:ibv:loss}
  \begin{aligned}
    L: \Sset \times \Dset & \;\to\; \Rplus,\\
    (f, d) & \;\mapsto\; \lVert \one_{\Gamma(f)} -d \rVert_{L^2(\Xset)}^{2},
  \end{aligned}
\end{equation}
where $\Dset \subset L^2(\Xset)$ is
the set of ``soft classification'' functions on~$\Xset$ (i.e.,
measurable functions defined on~$\Xset$ and taking values
in~$\left[ 0, 1 \right]$). %
Indeed, for all~$\nu \in \MGCX$ and $\xi \sim \nu$,
\begin{eqnarray*}
  {\overline L}_\nu(d) %
  = \Esp \left( L(\xi,d) \right) %
  = \lVert p_\nu - d \rVert_{ L^2(\Xset)}^{2} + \int p_\nu (1 - p_\nu)\, \dmu
\end{eqnarray*}
is minimal for $d = p_\nu$, and therefore
$\Hcal(\nu) = \inf_{d \in \Dset} {\overline L}_\nu (d)$.

The following theorem establishes the convergence of SUR (or
quasi-SUR) designs associated to this uncertainty functional using the
theory developed in Section~\ref{sec:SUR:ufunc-loss} for regular loss
functions.

\begin{theorem} \label{thm:p:un:moins:p}
  The loss function~\eqref{equ:ibv:loss} is regular in the sense of
  Definition~\ref{def:RegLoss}. As a consequence, all the conclusions of
  Corollary~\ref{cor:RegLoss} hold, and in
  particular $\Hcal(\Prob_n^\xi) \xrightarrow[]{\as}  0$
  for any quasi-SUR design associated with~$\Hcal$.
\end{theorem}

\newcommand \ProofItem[1] {%
  \medskip \noindent \textbf{#1}}

\begin{proof}
  The proof consists in six points, as follows:

  \ProofItem{a)} %
  $\Dset$ is separable

  The space $L^2(\Xset)$ is a separable metric space since $\Xset$ is
  a separable measure space (see, e.g., Theorem~4.13 in
  \cite{brezis2010functional}). Hence $\Dset$ is also separable.

  \ProofItem{b)} %
  for all~$d \in \Dset$, $L(\cdot, d)$ is $\Scal$-measurable

  Indeed,
  $f \mapsto \int_\Xset \left( \one_{f(x)\ge T} - d(x) \right)^2\,
  \mu(\dx)$ is $\Scal$-measurable by Fubini's theorem since the
  integrand is $\Scal \otimes \Bcal(\Xset)$-jointly measurable
  in~$(f, x)$.

  \ProofItem{c)} %
  for all $\nu \in \MGCX$, ${\overline L}_\nu$ takes finite
  values and is continuous on~$\Dset$

  Here ${\overline L}_\nu$ is clearly finite since the loss is
  upper-bounded by~$\mu(\Xset)$, and its continuity directly follows
  from the continuity of the norm.

  \ProofItem{d)} %
  $\Hcal = \Hcal_0 + \Hcal_1$, where
  $\Hcal_0(\nu) = \int_\Sset L_0\, \dnu$ for some
  $L_0 \in \cap_{\nu \in \MGCX} \mathcal{L}^1 \left( \Sset, \Scal, \nu
  \right)$, and $\Hcal_1$ is \UiFgcd

  Here this holds with $L_0 = 0$ and $\Hcal_1 = \Hcal$. %
  Indeed, $\Hcal$ is trivially \UiFgcd since the loss is upper-bounded.

  \ProofItem{e)} %
  $\Hcal_1$ is \AscFgcd

  Let $\xi \sim \GP (m, k)$ and let $(\nug_n)$ be a sequence of random
  measures $\nug_n \in \Pfrak(\xi)$ such that a.s.
  $\nug_n \to \nug_{\infty} \in \Pfrak(\xi)$.
  For $n \in \Nset \cup \{ \infty \}$, let $m_n$ and $k_n$ be the (random) mean
  and covariance functions of~$\nug_n$.
  For $u \in \Xset$ and $n \in \Nset \cup \{ \infty \}$, let also
  $\sigma^2(u) = k(u,u)$, $\sigma^2_n(u) = k_n(u,u)$, and
  \begin{equation*}
    g_n(u) = g \left( \bar{\Phi} \left( \frac{T - m_n(u)}{\sigma_n(u)}
      \right) \right),
  \end{equation*}
  where $g(p) = p (1 - p)$ and $\bar{\Phi}(t) = P( Z \geq t )$ where
  $Z$ is a standard Gaussian variable, with the convention that
  $\bar{\Phi}(0/0) = 1$.
  We will prove below that, for all $n \in \Nset \cup \{ + \infty \}$,
  \begin{equation}
    \label{eq:int_A}
    \Hcal( \nug_n ) = \int_{\Xset} g_n(u)\, \mu(\du) \eqas
    \int_A g_n(u)\, \mu(\du),
  \end{equation}
  where $A$ denotes the random subset of~$\Xset$ defined by
  \begin{equation*}
  A(\omega)=\{ u \in \Xset: \sigma(u)>0, \sigma_{\infty}(\omega,u) =
  0 , m_{\infty}(\omega,u) \neq T \}\cup \{ u \in \Xset:\; \sigma(u) >0,
  \sigma_{\infty}(\omega,u) > 0 \}.
  \end{equation*}
  The motivation for using~\eqref{eq:int_A} is that it is easy to
  prove the convergence of $g_n(\omega, u)$ for $u \in A(\omega)$ and
  that the set $A(\omega)$ does not depend on~$n$, which makes it
  possible to conclude using the dominated convergence theorem on the
  set $A(\omega)$ for almost all $\omega$.
  In more detail: since $\nug_n \to \nug_{\infty}$ almost surely, it
  holds for almost all $\omega \in \Omega$ that
  $m_n(\omega, \cdot) \to m_{\infty}(\omega, \cdot)$ and
  $\sigma_n(\omega, \cdot) \to \sigma_{\infty}(\omega, \cdot)$
  uniformly on~$\Xset$.
  Furthermore, for each $u \in A(\omega)$, either
  $\sigma_{\infty}(\omega, u) >0$ or
  $\sigma_{\infty}(\omega, u) =0, m_{\infty}(\omega, u) \neq T$.
  In both cases, we have that
  $g\left( \bar{\Phi}([m_n(\omega, u) - T]/ \sigma_n(\omega, u) )
  \right) \to g\left( \bar{\Phi}([m_{\infty}(\omega,
    u)-T]/\sigma_{\infty}(\omega, u)) \right)$.
  So, for almost all $\omega \in \Omega$ we can apply the dominated
  convergence theorem on $A(\omega)$ and thus obtain that
  \begin{equation*}
    \Hcal( \nug_n )= \int_A g_n(u)\, \mu(\du)
    \;\xrightarrow[n \to \infty]{\text{a.s.}}\;
    \Hcal( \nug_{\infty} )= \int_A g_{\infty}(u)\, \mu(\du),
  \end{equation*}
  which proves the claim.

  Let us now prove~\eqref{eq:int_A}.
  Observe first that, for any $u$ such that $\sigma(u) = 0$, we have
  $\sigma_n(u) \eqas 0$ for all $n \in \Nset \cup \{ \infty \}$ since
  $\nug_n \in \Pfrak(\xi)$.
  Hence, $g_n(u) \eqas 0$ when $\sigma(u) = 0$.
  [This is because of the convention $\bar{\Phi}(0/0) = 1$, which yields
  $g_n(u) = g( \mathbf{1}_{ m_n(u) \geq T } ) = 0$ when
  $\sigma_n(u) = 0$, regardless of whether $m_n(u) = T$ or not.]
  Thus, setting
  $B(\omega) = \{ u \in \Xset;\, \sigma(u) > 0,\,
  \sigma_{\infty}(\omega, u) = 0,\, m_{\infty}( \omega, u) = T \}$, we
  have for all $\omega \in \Omega$
  \begin{equation*}
    \Xset = \{ u \in \Xset, \sigma(u) = 0 \} \cup A( \omega ) \cup B( \omega ),
  \end{equation*}
  and the three sets of the right-hand side of the previous display
  are disjoint.
  Since, as discussed above, $g_n(u) \eqas 0$ for any $u$ such that
  $\sigma(u) = 0$, we obtain
  \begin{equation*}
    \Hcal(\nug_n) \eqas \int_A g_n(u)\, \mu(\du) + \int_B g_n(u)\, \mu(\du).
  \end{equation*}
  Thus, in order to prove~\eqref{eq:int_A}, it is sufficient
  to show that
  \begin{equation} \label{eq:to:show:int:bomega:zero}
    \int_B g_n(u)\, \mu(\du) \eqas 0.
  \end{equation}
  We will now establish~\eqref{eq:to:show:int:bomega:zero} by proving
  that, in fact, $\mu(B) = 0$ almost surely.  First, %
  since $(\omega, u) \mapsto m_\infty(\omega,u)$ and
  $(\omega, u) \mapsto \sigma_\infty(\omega,u)$ are jointly measurable
  (by continuity of~$m_\infty(\omega, \cdot)$ and
  $\sigma_\infty(\omega, \cdot)$ for all $\omega \in \Omega$), it
  follows from the Fubini-Tonelli theorem that
  \begin{equation} \label{eq:Domega}
    \Esp( \mu(B) ) =   \int_{\Xset}
    \one_{\sigma(u) > 0}\, \Esp( \one_{\sigma_{\infty}(u) = 0}
    \one_{m_{\infty}(u) = T} )\, \mu(\du).
  \end{equation}
  Then we have, for any $u \in \Xset$,
  \begin{align*}
    \Esp \left(  \one_{\sigma_{\infty}(u) = 0}
    (\xi(u) - m_{\infty}(u))^2 \right)
    & =
      \Esp \left( \Esp \left[ \left.
      \one_{\sigma_{\infty}(u) = 0} (\xi(u) - m_{\infty}(u))^2
      \right|
      \Fcal'_{\infty}
      \right]
      \right) \\
    & =
      \Esp \left( \one_{\sigma_{\infty}(u) = 0}\; \Esp \left[ \left.
      (\xi(u) - m_{\infty}(u))^2
      \right|
      \Fcal'_{\infty}
      \right]
      \right) \\
    & =
      \Esp \left( \one_{\sigma_{\infty}(u) = 0}\;
      \sigma_{\infty}^2(u)
      \right) \\
    & = 0,
  \end{align*}
  where $\Fcal'_{\infty}$ denotes the $\sigma$-algebra such that
  $\nug_\infty = \Prob\left( \xi \in \cdot \mid \Fcal'_\infty\right)$.
  Hence, the random variable
  $\one_{\sigma_{\infty}(u) = 0} (\xi(u) - m_{\infty}(u))^2$ is almost
  surely zero, since it is non-negative and has a zero
  expectation.
  Thus, for any $u \in \Xset$,
  the implication
  \begin{equation*}
    \sigma_{\infty}(u) = 0 \Longrightarrow
    \xi(u) = m_{\infty}(u)
  \end{equation*}
  holds almost surely. As a consequence, we have
  \begin{align*}
    \one_{\sigma(u) > 0}\, \one_{\sigma_{\infty}(u) = 0}
    \one_{m_{\infty}(u) = T}
    & =
      \one_{\sigma(u) > 0} \one_{\sigma_{\infty}(u) = 0} \one_{ \xi(u) = T} \\
    & \leq
      \one_{\sigma(u) > 0}  \one_{ \xi(u) = T}
    \\
    & = 0
  \end{align*}
  almost surely, since $\xi(u) \sim \mathcal{N} ( 0 ,
  \sigma(u)^2)$.
  Hence, the integrand in the right-hand side of~\eqref{eq:Domega} is zero,
  which implies that $\Esp\left( \mu(B) \right) = 0$, and therefore
    $\mu(B) \eqas 0$ since $\mu(B)$ is a non-negative random variable.
  Thus~\eqref{eq:int_A} holds and the proof of {\bf e)} is complete.

  \ProofItem{f)} 
  $\Zset_\Hcal = \Zset_\Gcal$

  Let $\nu \in \Zset_{\Gcal}$ and let $\xi \sim \nu$. Let
  $m,k,\sigma^2$ be defined as above.
  Let $U \sim \mathcal{N}(0,1)$ be independent of~$\xi$.
  Since $\Gcal(\nu) = 0$, we have from the law of
  total variance
  \begin{eqnarray*}
    \int_{\Xset} \var \left( \Esp
    \left( \one_{\xi(u) \geq T} | Z_x \right) \right) \mu(\du) = 0
  \end{eqnarray*}
  for all $x \in \Xset$, where $Z_x = \xi(x) + \tau(x) U$.
  Hence, for all $x \in \Xset$, for almost all $u \in \Xset$, we have
  \begin{equation*}
    \var \left(
      \bar{\Phi} \left(
        \frac{%
          T -  m(u) - \frac{k (x, u)\, (Z_x - m(x))}{\sigma^2(x) + \tau^2(x)}%
        }{%
          \sqrt{   \sigma^2(u) - \frac{k(x,u)^2}{\sigma^2(x)+\tau^2(x)}}
        }
      \right)
    \right) = 0,
  \end{equation*}
  which implies that $k (x, u) = 0$ (as can be proved without
  difficulty by separating the cases of nullity and non-nullity of the
  denominator).
  Thus, if there exists $x^*$ for which
  $\sigma^2(x^*) = k(x^*, x^*) > 0$, we obtain a contradiction, since
  then $k (x, u) > 0$ in a neighborhood of~$x^*$ by continuity.
  We conclude that $\sigma^2(x) = 0$ for all $x \in \Xset$, and
  therefore $\Hcal(\nu) = 0$.
\end{proof}

In the next proposition, we refine Theorem~\ref{thm:p:un:moins:p} by
showing that it entails a consistent estimation of the excursion set
$\Gamma(\xi)$.

\begin{proposition} \label{prop:consist:p:un:moins:p}
  For any quasi-SUR design associated with $\Hcal$, as $n \to \infty$,
  almost surely and in~$L^1$,
  \begin{equation*}
    \int_{\Xset} \left( \one_{\xi(u) \geq T} - p_n(u) \right)^2 \mu(\du) \to 0
  \end{equation*}
  and
  \begin{equation*}
    \int_{\Xset} \left( \one_{\xi(u) \geq T}
      - \one_{p_n(u) \geq 1/2} \right)^2 \mu(\du) \to 0.
  \end{equation*}
\end{proposition}

\begin{proof}
  From steps \textbf{e)} and \textbf{f)} in the proof of
  Theorem~\ref{thm:p:un:moins:p}, it follows that
  \begin{equation*}
    \int_{\Xset} \left( \one_{\xi(u) \geq T} - p_n(u) \right)^2 \mu(\du) \eqas
    \int_A \left( \one_{\xi(u) \geq T} - p_n(u) \right)^2 \mu(\du).
  \end{equation*}
  Also, for almost all~$\omega \in \Omega$ and all
  $u \in A( \omega )$,
  $p_{n}(\omega, u) \to \one_{\xi(\omega, u) \geq T} $ as
  $n \to \infty$ since $\sigma_{\infty} \equiv 0$ a.s.\ from the proof
  of \textbf{f)} in Theorem~\ref{thm:p:un:moins:p} and the conclusion
  of this theorem.
  Hence the first part of the proposition follows by applying the
  dominated convergence theorem twice.
  The proof of the second part of the proposition is identical.
\end{proof}

\subsection{The variance of excursion volume functional}
\label{sec:example:variance}

Following up on the example of Section~\ref{sec:example:pUnMoinsp}, we
consider now the alternative measure of residual uncertainty
$H_n = \var_{n}(\alpha(\xi))$ from \cite{bect12, chevalier14}; in
other words, we consider the uncertainty functional
\begin{equation}
  \label{eq:uncertainty_var}
  \Hcal(\nu) = \int_\Sset \left( \alpha(f)
    - \overline{\alpha}_\nu \right)^2\, \nu(\df),
\end{equation}
where $\overline{\alpha}_\nu = \int_\Sset \alpha\, \dnu$.
The corresponding sampling criterion is
\begin{equation*}
  J_{n}(x) = \Esp_{n, x} \bigl( \var_{n+1}\left( \alpha(\xi) \right) \bigr).
\end{equation*}
This uncertainty functional again derives from a loss function: indeed,
$L(f,d)=(\alpha(f) -d)^2$ with $\Dset=\Rset$ leads to
\begin{eqnarray*}
  {\overline L}_{\Prob_{n}^{\xi}}(d)
  =\Esp_{n}[(\alpha(\xi)-d)^2]
  =\var_{n}(\alpha(\xi)) + (\Esp_{n}(\alpha(\xi))-d)^2,
\end{eqnarray*}
where ${\overline L}_{\Prob_{n}^{\xi}}$ reaches its infimum for
$d=\Esp_{n}(\alpha(\xi))$, and therefore
$H_{n}=\inf_{d \in \Dset} {\overline L}_{\Prob_{n}^{\xi}}(d)$.
As in the previous section, consistency is established in the
following theorem by proving that the loss function~$L$ is regular.

\begin{theorem} \label{thm:var}
  The loss function $L (f, d) = (\alpha(f) -d)^2$, where
  $d \in \Dset = \Rset$, is regular in the sense of
  Definition~\ref{def:RegLoss}. As a consequence, all the conclusions of
  Corollary~\ref{cor:RegLoss} hold, and in
  particular $\Hcal(\Prob_n^\xi) \xrightarrow[]{\as}  0$
  for any quasi-SUR design associated with~$\Hcal$.
\end{theorem}

\begin{proof}
  The proof consists of the same six points as the proof of
  Theorem~\ref{thm:p:un:moins:p}. Here \textbf{a)} and \textbf{c)} are
  obvious. Let us now prove the four remaining points.

  \ProofItem{b)} %
  Since
  $L(f,d) = ( \int_{\Xset} \one_{f(u) \geq T}\, \mu(\du) - d )^2$, it
  can be shown similarly as in the proof of
  Theorem~\ref{thm:p:un:moins:p} that, for any fixed $d \in \Dset$,
  $L(f,d)$ is an $\Sset$-measurable function of $f$.

  \ProofItem{d)} %
  We use again $L_0 = 0$ for this criterion. %
  The functional $\Hcal_1 = \Hcal$ is trivially \UiFgcd since
  $0 \le \Hcal_1 \le \mu(\Xset)^2$.

  \ProofItem{e)} %
  Let us now show that $\Hcal_1 = \Hcal$ is \AscFgcd.
  Let $\xi$ denote a random element in~$\Sset$,
    $\xi \sim \GP (m, k)$, and let $(\nug_n)$ be a sequence of random
    measures $\nug_n \in \Pfrak(\xi)$ such that
    $\nug_n \to \nug_{\infty} \in \Pfrak(\xi)$ almost surely.
  For all $n \in \Nset \cup \{ + \infty \}$, let
    $\Fcal'_n$ denote a $\sigma$-algebra such that
    $\nug_n = \Prob\left( \xi \in \cdot \mid \Fcal'_n \right)$.
    Then, using Fubini's theorem, we can rewrite~$\Hcal(\nug_n)$ as
  \begin{equation*}
    \Hcal(\nug_n) = \int_{\Xset} \int_{\Xset}
    c_n( u_1, u_2 )\, \mu(\du_1)\, \mu(\du_2),
  \end{equation*}
  where
  $c_n(u_1,u_2) = \cov\left( \one_{\xi(u_1) \geq T}, \one_{\xi(u_2)
      \geq T} \mid \Fcal'_n \right)$ for $n \in \Nset \cup \{ + \infty \}$.
  Consider the partition
  \begin{equation*}
    \Xset = \{ u \in \Xset, \sigma(u) = 0 \} \cup A( \omega ) \cup B( \omega ),
    \qquad \omega \in \Omega,
  \end{equation*}
  where $\sigma$, $A$ and~$B$ are defined
  as in the proof of Theorem~\ref{thm:p:un:moins:p}.
  Recalling from step {\bf e)} of this proof that
  $\var\left( \one_{\xi(u) \geq T} \mid \Fcal'_n \right) \eqas 0$ when
  $\sigma(u) = 0$, and that $\mu(B) \eqas 0$, %
  we obtain that, for all $n \in \Nset \cup \{ + \infty \}$,
  \begin{equation*}
    \Hcal(\nug_n) \eqas %
    \int_A \int_A
    c_n(u_1, u_2)\, \mu(\du_1)\, \mu(\du_2).
  \end{equation*}
  For $j=1,2$ and $u_j \in A(\omega)$, we have either
  $\sigma_{\infty}(u_j) >0$ or
  $\sigma_{\infty}(u_j) =0, m_{\infty}(u_j) \neq T$.
  Hence, for almost all $\omega \in \Omega$, for $u_1 \in A(\omega)$
  and $u_2 \in A(\omega)$, we obtain
  $c_n(u_1,u_2) \to c_{\infty}(u_1,u_2)$ by
  the following lemma (proved later):
  \begin{lemma} \label{lem:convergence:cov}
    Let $m_n = (m_{n1},m_{n2})^t \to (m_1,m_2)^t = m$ as
    $n \to \infty$.
    Consider a sequence of covariance matrices $\Sigma_n$ such that
    \begin{equation*}
      \Sigma_n =
      \begin{pmatrix}
        \sigma_{n1} & \sigma_{n12} \\
        \sigma_{n12} & \sigma_{n2}
      \end{pmatrix}
      \xrightarrow[n \to \infty]{}
      \begin{pmatrix}
        \sigma_{1} & \sigma_{12} \\
        \sigma_{12} & \sigma_{2}
      \end{pmatrix}
      = \Sigma.
    \end{equation*}
    Assume that for $i=1,2$ we have $m_i \neq T $ or $\sigma_i
    >0$.
    Let $Z_n \sim \mathcal{N}(m_n,\Sigma_n)$ and
    $Z \sim \mathcal{N}(m,\Sigma)$.
    Then as $n \to \infty$,
    $\cov( \one_{\{Z_{n1} \geq T\}},\one_{\{Z_{n2} \geq T\}} ) \to
    \cov( \one_{\{Z_{1} \geq T\}},\one_{\{Z_{2} \geq T\}} )$.
  \end{lemma}
  Finally, using the dominated convergence theorem on
  $A(\omega) \times A(\omega)$ for almost all $\omega \in \Omega$, we
  conclude that $\Hcal(\nug_n) \to \Hcal(\nug_{\infty})$ almost
  surely, which proves that $\Hcal$ is \AscFgcd.

  \ProofItem{f)} 
  Let $\nu \in \MGCX$ and let $\xi \sim \nu$.
  Let also $Z_x = \xi(x) + \tau(x)\, U$, with $U \sim \Ncal(0,1)$
  independent of~$\xi$, so that
  $\Jcal_x(\nu) = \Esp \left( \var\left( \alpha(\xi) \mid Z_x
    \right)\right)$.
  We first remark that, from~\eqref{eq:def-Gx} and the
  law of total variance, for any $x \in \Xset$,
  \begin{equation}
    \Gcal_x(\nu) %
    = \var( \alpha(\xi) ) - \Esp(  \var(  \alpha(\xi) | Z_x )  )
    = \var(  \Esp(  \alpha(\xi) | Z_x )  ).
    \label{eq:varalpha:Gx}
  \end{equation}
  Then we have the following sequence of equivalences:
  \begin{align} \def\iff{\quad\Leftrightarrow\quad}
    \Gcal(\nu) = 0 %
    & \iff \forall x \in \Xset,\; %
      \Gcal_x(\nu) = 0
      \nonumber \\
    & \iff \forall x \in \Xset,\; %
      \var\left( \Esp(\alpha(\xi) \mid Z_x) \right) = 0
      \nonumber \\
    & \iff \forall x \in \Xset,\; %
      \alpha(\xi) - \Esp(\alpha(\xi)) \,\perp\, L^2(Z_x)
      \label{equ:ortho:alpha-L2Zx}.
  \end{align}
  The first equivalence follows directly from the definition
  of~$\Gcal$: $\Gcal (\nu) = \sup_{x\in\Xset}\, \Gcal_x (\nu)$, since
  $\Gcal_x (\nu)$ is non-negative for all $x \in \Xset$, the second
  one from~\eqref{eq:varalpha:Gx}, and the third one from the fact
  that $\Esp(\alpha(\xi) \mid Z_x)$ is the orthogonal projection
  of~$\alpha(\xi)$ onto~$L^2(Z_x)$.

  Let now~$\nu \in \Zset_\Gcal$.
  Using Lemma~\ref{lem:ortho-sum}, it follows
  from~\eqref{equ:ortho:alpha-L2Zx} that
  $\alpha(\xi) - \Esp(\alpha(\xi)) \,\perp\, L^2(\xi(x))$, for all
  $x \in \Xset$.
  In particular,
  $\alpha(\xi) - \Esp(\alpha(\xi)) \,\perp\, \one_{\xi(x) \geq T}$,
  for all $x \in \Xset$, and thus
  \begin{equation*}
    \var\left( \alpha(\xi) \right)
    \;=\; \int \cov\left( \alpha(\xi), \one_{\xi(x) \geq T} \right)\, \mu(\dx)
    \;=\; 0,
  \end{equation*}
  which concludes the proof.
\end{proof}

\begin{proof}[Proof of Lemma~\ref{lem:convergence:cov}]
  By the convergence of moments and Gaussianity,
  $(Z_{n1},Z_{n2})$ converges in distribution to~$(Z_1,Z_2)$.
  Furthermore, from the assumptions the cumulative distribution
  functions of $Z_1$ and $Z_2$ are continuous at $T$, which implies
  that, by the Portemanteau theorem,
  $P( Z_{ni} \geq T) \to P( Z_{i} \geq T)$.
  In addition, $Y:=\min(Z_{1},Z_{2})$ also has a continuous cumulative
  distribution function at~$T$ and, as
  $\Esp( \one_{\{Z_{1} \geq T\}} \one_{\{Z_{2} \geq T\}} )=P(Y\geq
  T)$, we get similarly that
  $\Esp( \one_{\{Z_{n1} \geq T\}} \one_{\{Z_{n2} \geq T\}} ) \to \Esp(
  \one_{\{Z_{1} \geq T\}} \one_{\{Z_{2} \geq T\}} )$, which completes
  the proof.
\end{proof}

Similarly as before, in the next proposition, we show that
Theorem~\ref{thm:var} yields a consistent estimation of the excursion
volume.

\begin{proposition} \label{prop:consist:var}
  For any quasi-SUR design associated with $\Hcal$, as
  $n \to \infty$, almost surely and in $L^1$,
  $\Esp_n(\alpha(\xi))\to \alpha(\xi)$.
\end{proposition}

\begin{proof}
  Let $\alpha = \alpha(\xi)$.
  We know from, e.g., Theorem~6.23 in~\cite{kall02},
    that $\Esp_n(\alpha) \to \Esp_\infty(\alpha)$ almost surely and
    in~$L^1$.  Moreover, it follows from Theorem~\ref{thm:var} that
  $\var_n(\alpha) \to 0$ almost surely, and therefore
  $\Esp( \var_n(\alpha) ) \to 0$ by dominated
  convergence.
  Hence,
  $\Esp( \Esp_n [ ( \Esp_n(\alpha) - \alpha )^2 ] ) \to 0$,
  which shows that $\Esp_n(\alpha)$ converges
  in~$L^1$, and thus almost surely as well,
  to~$\Esp_\infty(\alpha) \eqas \alpha$
\end{proof}

\begin{remark}
  Contrary to the case of the integrated Bernoulli variance functional
  in Proposition~\ref{prop:consist:p:un:moins:p}, it is not possible
  to prove that a SUR sequential design associated with the
  uncertainty functional~\eqref{eq:uncertainty_var} results in a
  consistent estimation of the set
  $ \Gamma( \xi ) = \{u \in \Xset: \xi(u) \geq T\} $.
  Indeed, for instance, let $\Xset = [-1,1]$, let $\mu$ be Lebesgue
  measure, let $T=0$ and let $\xi$ have zero mean function and
  covariance function~$k$ defined by $k(u,v) = uv$ for
  $u,v \in [-1,1]$.
  Then, almost surely, the set $\Gamma( \xi )$ is equal to $[-1,0]$ or
  to $[0,1]$ (with probabilities $1/2$ for both cases).
  Hence, we have, with $\nu$ the distribution of $\xi$,
  $\Hcal (\nu) = 0$ because $\alpha (\xi) = 1$ almost surely. Thus,
  any sequential design $(X_n)_{n \geq 1}$ is a SUR sequential design
  associated with~\eqref{eq:uncertainty_var}, since we have
  $J_n(x) = 0$ for any $x \in \Xset$ and $n \geq 1$.
  However, the conclusions of
  Proposition~\ref{prop:consist:p:un:moins:p} clearly do not hold for
  any sequential design.
  For instance, if $X_n = 0$ for all $n \geq 1$, we have
  $p_n(u) = 1/2$ for all $u \in \Xset$.

  As a conclusion, for the SUR strategy associated with the
  uncertainty functional~\eqref{eq:uncertainty_var}, and thus based on
  $\mu ( \Gamma (\xi) )$, it can only be guaranteed that
  $\mu ( \Gamma (\xi) )$, but not $\Gamma(\xi)$ in general, is
  estimated consistently.
\end{remark}

\subsection{The knowledge gradient functional}
\label{sec:example:knowledge}

Coming to the topic of sequential design for global optimization, we
now focus on the knowledge gradient criterion \cite{frazier08,
frazier09, Scott.etal2011}, which is an extension to
the general (noisy) case of the strategy proposed in the 70's
by~\cite{mockus78} for the noiseless case.
We shall consider, here and in the next section, the case of a
maximization problem.
The knowledge gradient sampling criterion, to be
  maximized, is then defined by
\begin{equation}
  \label{equ:KG:JnTilde}
  G_{n}(x) = \Esp_{n,x} \left( \max m_{n+1} \right) - \max m_n,
\end{equation}
with maxima taken over the whole domain~$\Xset$ as in \cite{frazier08,
  frazier09} for the case of a discrete~$\Xset$, and in Section~3
of~\cite{Scott.etal2011} for the case of a ``continuous''~$\Xset$;
we do not consider the KGCP approximation introduced in Section~4
of~\cite{Scott.etal2011}.

\begin{remark} \label{rem:KG:constant-term}
  The quantity~$\max m_n$ in~\eqref{equ:KG:JnTilde} does not depend on
  the sampling point $x$, and thus plays no part in the selection of
  the next observation point.
  The motivation for writing $G_n(x)$ in this form is that the
  sampling criterion thus defined is
  non-negative, and becomes equal to zero when $\sigma_n \equiv 0$.
  The first term in the right-hand side of~\eqref{equ:KG:JnTilde} is
  exactly the sampling criterion proposed by~\cite{mockus78} in the
  noiseless case.
\end{remark}

The following criterion, to be minimized, clearly defines the same
strategy as~\eqref{equ:KG:JnTilde}:
\begin{align*}
  J_{n}(x)
  & = \Esp_{n}\left( \max \xi \right)
    - \Esp_{n,x} \left( \max m_{n+1} \right)\\
  & = \Esp_{n,x} \left( \Esp_{n+1}(\max\xi) -  \max m_{n+1} \right),
\end{align*}
and clearly appears, under the second form, as the SUR sampling
criterion corresponding to the uncertainty functional
\begin{equation}
  \label{eq:uncertainty_kg}
  \Hcal(\nu) = \int_\Sset \max f\, \nu(\df) - \max m_\nu.
\end{equation}
Moreover, the original sampling criterion~\eqref{equ:KG:JnTilde} is
easily seen to be the value $G_n(x) = \Gcal_x\left( \Pnxi \right)$ of
the associated expected gain functional.

This time again, the uncertainty functional~$\Hcal$ derives from a
loss function, with $\Dset=\Xset$ and $L(f,d) = \max f - f(d)$,
leading to
\begin{eqnarray*}
  {\overline L}_{\Prob_{n}^{\xi}}(d)
  = \Esp_{n}\left( \max \xi \right) - m_n(d).
\end{eqnarray*}
The average loss ${\overline L}_{\Prob_{n}^{\xi}}$ reaches its infimum for
$d \in \argmax m_n$, and so
$H_{n}=\inf_{d \in \Dset} {\overline L}_{\Prob_{n}^{\xi}}(d)$.
Following the same route as in the last two sections, we have:

\begin{theorem} \label{thm:knowledge:gradient}
  The loss function $L (f, d) = \max f - f(d)$,
  where $d \in \Dset=\Xset$, is regular in the sense of
  Definition~\ref{def:RegLoss}. As a consequence, all the conclusions of
  Corollary~\ref{cor:RegLoss} hold, and in
  particular $\Hcal(\Prob_n^\xi) \xrightarrow[]{\as}  0$
  for any quasi-SUR design associated with~$\Hcal$.
\end{theorem}

\begin{proof}
  The proof consists in the same six points as in the proof of
  Theorem~\ref{thm:p:un:moins:p}.

  \ProofItem{a)} %
  $\Xset$ is a compact metric space, hence separable.

  \ProofItem{b)} %
  The mapping $L(\cdot, d): f \mapsto \max f - f(d)$ is continuous
  on~$\Sset$, hence $\Scal$-measurable.

  \ProofItem{c)} %
  ${\overline L}_\nu: d \mapsto \int \max f\, \nu(\df) - m_\nu(d)$ is
  continuous since $m_\nu \in \Sset$ for all~$\nu \in \MGCX$.

  \ProofItem{d)} %
  Let $L_0(f) = \max f$.
  Since $\Xset$ is compact, it holds for any Gaussian
  measure $\nu \in \MGCX$ and any $\xi \sim \nu$ that
  $\Esp( \max_\Xset |\xi | ) < \infty$ (see
    Section~\ref{sec:Gauss-RE-RM}), and thus we have
  $L_0 \in \cap_{\nu \in \MGCX} \mathcal{L}^1 \left( \Sset, \Scal, \nu
  \right)$.
  Moreover, it follows from Proposition~\ref{prop:UiFgcd:upperbound}
  that $\Hcal_1:\nu \mapsto - \max m_\nu$ is \UiFgcd, since
  $\left| \Hcal_1 (\nu) \right| \le \int L^+ \dnu$ with
  $L^+(f) \eqdef \max |f|$, and $L^+ \in \cap_{\nu \in \MGCX} \mathcal{L}^1
  \left( \Sset, \Scal, \nu \right)$.

  \ProofItem{e)} 
  $\Hcal_1: \nu \mapsto - \max m_\nu$ is continuous, hence \AscFgcd.
    Indeed, consider a sequence of measures $\nu_n \in \MGCX$ converging to
  a limit~$\nu_{\infty} \in \MGCX$ in the
  sense of Definition~\ref{def:topology:Mzero}.
  Then $m_{\nu_n}$ converges uniformly to $m_{\nu_\infty}$ as $n \to \infty$,
  and therefore
  $\Hcal_1(\nu_n) = - \max m_{\nu_n}$ converges  to
  $\Hcal_1(\nu_{\infty}) = - \max m_{\nu_\infty}$
  by continuity of $f \mapsto \max f$ on~$\Sset$.

  \ProofItem{f)} 
  Let $\nu \in \Zset_{\Gcal}$ and let $\xi \sim \nu$. Let
  $m,k,\sigma^2$ be defined, w.r.t. $\xi$, as in the proof of
  Theorem~\ref{thm:p:un:moins:p}.
  Let $Z_x = \xi(x) + \tau(x) U$ with $U \sim \mathcal{N}(0,1)$
  independent of~$\xi$.
  Let $x^* \in \argmax m$.
  We have, for all $x \in \Xset$,
  \begin{align}
    0
    & = \Gcal_x(\nu) = \Hcal(\nu) - \Jcal_x(\nu)\nonumber\\
    & = \left(
      \Esp\left( \max \xi \right) - \max m \right)
      -
      \Esp\Bigl(
      \Esp\left( \max \xi \mid Z_x \right)
      -
      \max_{u \in \Xset} \Esp\bigl( \xi(u) \mid Z_x \bigr)
      \Bigr) \nonumber\\
    & = \Esp\Bigl( \max_{u \in \Xset} \Esp\bigl( \xi(u) \mid Z_x \bigr) \Bigr) -
        m(x^*), \label{equ:zorglub}
  \end{align}
  by the law of total expectation and the optimality property of $x^*$.
  For all $x,y \in \Xset$ it holds that
    \begin{equation*}
    m(x^*) = \Esp\Bigl( \Esp( \xi(x^*) \mid Z_x) \Bigr)
    \le \Esp\Bigl(
    \max\bigl( \Esp( \xi(y) \mid Z_x),\, \Esp( \xi(x^*) \mid Z_x) \bigr) \Bigr)
    \le \Esp\Bigl( \max_{u \in \Xset} \Esp\bigl( \xi(u) \mid Z_x \bigr) \Bigr),
  \end{equation*}
  and therefore, using~\eqref{equ:zorglub},
  \begin{equation*}
    0 \le \Esp\Bigl( \max ( \Esp( \xi(y) | Z_x) ,  \Esp( \xi(x^*) | Z_x) ) -
      \Esp( \xi(x^*) | Z_x) \Bigr)
    \le \Gcal_x(\nu) = 0.
  \end{equation*}
  Setting
  $W_{x,y} \eqdef \Esp( \xi(y) \mid Z_x) - \Esp( \xi(x^*) \mid Z_x)$,
  we have thus proved that
  $\Esp\left( \max\left( 0,\, W_{x,y} \right) \right) = 0$, %
  from which it follows that $\var\left( W_{x,y} \right) = 0$ since
  \begin{equation*}
    W_{x,y} = m(y) - m(x^*) +  \one_{ \sigma^2(x) + \tau^2(x) >0 }
    \frac{k(x,y) - k(x,x^*)}{\sigma^2(x) + \tau^2(x)}
    (Z_x - m(x))
  \end{equation*}
  is Gaussian.  Observe now that%
  \begin{equation*}
      \var\left( W_{x,y} \right) =
      \one_{ \sigma^2(x) + \tau^2(x) >0}\;
      \frac{(k(x,y) - k(x,x^*))^2}{\sigma^2(x) + \tau^2(x)}.
  \end{equation*}
  Hence it must be the case that either $\sigma^2(x) + \tau^2(x) = 0$
  or $k(x,y) = k(x,x^*)$.
  But, if $\sigma^2(x) + \tau^2(x) = 0$ then $\sigma(x) = 0$ and
  therefore $k(x,y) = k(x, x^*) = 0$.
  Summing up, we have proved that
  \begin{equation*}
    k(x,y) = k(x, x^*), \qquad \forall x, y \in \Xset.
  \end{equation*}
  As a consequence, for all $x,y\in \Xset$, we have
  \begin{equation*}
    k(x,x) - k(x,y)
    = k(x,x^*) - k(x,x^*) = 0,
  \end{equation*}
  and therefore
  \begin{equation*}
    \var\left( \xi(x) - \xi(y) \right)
    = \left( k(x, x) - k(x,y) \right)
    + \left( k(y, y) - k(x,y) \right)
    = 0.
  \end{equation*}
  It follows that, almost surely, the sample paths of~$\xi - m$ are
  constant over~$\Xset$, and so
  $\max \xi = \xi(x^*) - m(x^*) + \max m = \xi(x^*)$.
  We have thus proved that
  $\Hcal(\nu) = \Esp\left( \max \xi \right) - m(x^*) =
  0$ for any $\nu \in \Zset_\Gcal$, which concludes the proof.
\end{proof}

In the next proposition, we refine
Theorem~\ref{thm:knowledge:gradient} by showing that the
loss~$\max \xi - \xi(X_n^*)$ goes to zero for any sequence of optimal
decisions $X_n^* \in \argmax m_n$.

\begin{proposition} \label{prop:consist:var}
  Let $\left( X^*_n \right)$ be any sequence of $\Fcal_n$-measurable
  $\Xset$-valued random variables such that $X_n^* \in \argmax m_n$
  almost surely for all~$n$.
  Then, for any quasi-SUR design associated with $\Hcal$,
  $\xi(X_n^*) \to \max \xi$ almost surely and in $L^1$.
\end{proposition}

\begin{proof}
  From step~\textbf{f)} in the proof of
  Theorem~\ref{thm:knowledge:gradient}, and the fact that
  $\Hcal(P_{\infty}^\xi) \eqas 0$, it follows that
  the sample paths of~$\xi - m_\infty$ are almost
    surely constant over~$\Xset$.
    Let $X^*$ denote an $\Xset$-valued random variable such that
    $X^* \in \argmax \xi$.
    Then, we have
  \begin{equation*}
    \limsup_{n \to \infty} \left( \xi(X^*) - \xi(X_n^*) \right)
    = \limsup_{n \to \infty} \left( m_{\infty}(X^*) - m_{\infty}(X_n^*) \right)
    = \limsup_{n \to \infty} \left( m_{n}(X^*) - m_{n}(X_n^*) \right)
    \leq 0
  \end{equation*}
  almost surely.  This implies, since $\xi(X^*) - \xi(X_n^*) \ge 0$,
  that $\xi(X_n^*) \to \xi(X^*) = \max \xi$ almost surely.
  Convergence in the $L^1$ sense is finally obtained by the dominated
  convergence theorem.
\end{proof}

\subsection{The expected improvement functional}
\label{sec:example:ei}

This section addresses the celebrated expected improvement strategy
\cite{mockus78, EGOEBBF}\footnote{%
  As explained in Section~\ref{sec:example:knowledge}, the strategy
  originally proposed by~\cite{mockus78}---and earlier work in Russian
  by J.~Mockus and A.~\v{Z}ilinskas---is more accurately described, in
  principle, as a special (noiseless) case of what we have called the
  ``knowledge gradient'' strategy.
  However, for the particular Brownian motion-based Gaussian process
  prior used in~\cite{mockus78}, the maximum of~$m_n$ always occurs at
  an observation point, and the criterion of~\cite{mockus78} then
  coincides (see page~21 of the paper) with what is, currently, commonly
  referred to as the expected improvement criterion.
}.
Assume that exact (noiseless) evaluations can be made, in other words,
that $\tau(x) = 0$ for all $x \in \Xset$:
then, we define the expected improvement criterion, to be maximized, as
\begin{equation} \label{eq:Gn:EI}
  G_n(x) = \Esp_{n,x} \left( M_{n+1} - M_n \right),
\end{equation}
with $M_n = \max_{x \in \Xset:\, \sigma_n(x) = 0}\, \xi(x)$ %
and $\sigma_n^2(x) = k_n(x, x)$.
Observe that, on the right-hand side of~\eqref{eq:Gn:EI}, similarly to
Remark~\ref{rem:KG:constant-term} for the knowledge gradient
criterion, only $M_{n+1}$ actually depends on the new observation
point~$X_{n+1} = x$.
Note also that we need at least one~$x \in \Xset$ such that~$\sigma_n(x) = 0$
for~$M_n$ to be well defined, which is always true as soon as~$n \ge 1$
(in practice, \eqref{eq:Gn:EI} is typically used after an initial
design of size~$n_0 > 0$).

\begin{remark} \label{rem:unusual-EI}
  Our definition of the EI strategy, and in particular of the current
  best value~$M_n$, differs slightly from the usual one
  \citep{EGOEBBF}, which takes
  $M_n = \max\left( \xi(X_1), \ldots,\, \xi(X_n) \right)$.
  This minor variation is necessary if we want to see the EI
  strategy as stemming from some uncertainty functional.
  Remark that, in the case of a non-degenerate Gaussian process (i.e.,
  when $\sigma_n(x) = 0$ if and only if
  $x \in \left\{ X_1,\, \ldots,\, X_n \right\}$), the two definitions
  of~$M_n$ coincide and the criterion can be written more familiarly as
  \begin{equation} \label{eq:Gn:EI:more:familiar} %
    G_n(x) = \Esp_n \left( \max\left( 0,\, \xi(x) - M_n \right) \right).
  \end{equation}
  The sampling criteria~\eqref{eq:Gn:EI}
  and~\eqref{eq:Gn:EI:more:familiar} no longer agree in general, since
  it can happen for degenerate Gaussian processes that
  $M_n > \max\left( \xi(X_1),\, \ldots,\, \xi(X_n) \right)$.
  Degeneracy occurs, e.g., in the case of finite-dimensional Gaussian
  processes (i.e., linear models with a Gaussian prior on their
  coefficients), or for processes with pathwise invariance properties
  \citep{ginsb:sym} (for instance $\xi(x) = - \xi(-x)$ for
  all~$x\in \Xset$, almost surely).
\end{remark}

It turns out that the sequential design obtained by iteratively
maximizing~\eqref{eq:Gn:EI} can be interpreted as a SUR sequential
design.
Indeed, we have
\begin{equation*}
  \Esp_{n,x} \left( M_{n+1} - M_n \right)
  = \Esp_n\left( \max\xi - M_n \right)
  - \Esp_{n,x}\left( \max\xi - M_{n+1} \right),
\end{equation*}
where the subscript ``$x$'' has been dropped from the first
expectation since its argument does not depend on the
position~$X_{n+1}$ of the next evaluation.
Thus, using the fact that
$\Esp_{n,x} \left( \max\xi \right) = \Esp_{n,x} \left(
  \Esp_{n+1}\left( \max\xi \right) \right)$ by the law of total
expectation, we see that maximizing~\eqref{eq:Gn:EI} is equivalent
to minimizing
\begin{equation*}
  J_n(x) = \Esp_{n,x} \bigl( \Esp_{n+1}\left( \max\xi \right) - M_{n+1} \bigr),
\end{equation*}
which is precisely the SUR strategy associated with the uncertainty
functional defined, for any $\nu \in \MGCX$ such that $\sigma_\nu$
vanishes at at least one~$x \in \Xset$, by
\begin{equation}
  \Hcal(\nu) = \int_\Sset \max f\; \nu(\df)
  - \max_{x \in \Xset: \sigma_\nu(x)=0}\, m_\nu(x).
  \label{eq:EI:uncertainty-func}
\end{equation}
  Indeed,
  $\max_{x \in \Xset:\, \sigma_\nu(x) = 0}\, \xi(x) \eqas
  \max_{x \in \Xset:\, \sigma_\nu(x) = 0}\, m_\nu(x)$
  for any such~$\nu$ and any~$\xi \sim \nu$.
  Moreover, the expected improvement criterion~$G_n$ turns out to
  be the value~$G_n(x) = \Gcal_x\left( \Pnxi \right)$ of the associated
  expected gain functional.
  In order to have an uncertainty function~$\Hcal$ that is well defined
  on all of~$\MGCX$, set
  \begin{equation}
    \Hcal(\nu) = \int_\Sset \left( \max f - \min f \right)\; \nu(\df)
    \label{eq:EI:uncertainty-func-2}
  \end{equation}
  for all~$\nu \in \MGCX$ such that~$\sigma_\nu$ does not vanish.

The uncertainty
  functional~\eqref{eq:EI:uncertainty-func}--\eqref{eq:EI:uncertainty-func-2}
  is measurable (see Proposition~\ref{prop:EI:H:measurable}) and can
be associated with a certain loss function, as shown
  in the following result.
Contrary to the case of the three previous criteria, however, this
loss function is not regular in general.

\begin{proposition} \label{prop:EI:not:regular}
  The EI uncertainty functional is of the
  form~\eqref{eq:def-H-from-loss}, with $L$ the loss function defined
  on the decision space
  $\Dset = \Xset \times \left( \Rset \cup \{-\infty\} \right)$, %
  for any~$f \in \Sset$ and $d = (x^*,z^*) \in \Dset$, by
  \begin{equation} \label{eq:loss:EI} %
    L \left(f, d\right) =
    \begin{cases}
      \max f - z^* & \text{if } f(x^*) = z^* \text{ and } z^* > -\infty,\\
      \max f - \min f & \text{if } z^* = - \infty,\\
      +\infty & \text{otherwise.}
    \end{cases}
  \end{equation}
  Assuming that $\Xset \subset \Rset^p$ has a non-empty interior, the
  loss function~\eqref{eq:loss:EI} is not regular, and neither would
  be any other loss function that could be associated with the EI uncertainty
  functional.
\end{proposition}

\begin{proof}
  Let us first prove that~$\Hcal$ is of the form~\eqref{eq:def-H-from-loss}.
  Let $\nu \in \MGCX$, $\xi \sim \nu$ and $d = (x^*,z^*) \in \Dset$.
  Then the average loss is equal to
  $\overline{L}_{\nu} (d) = \Esp\left( \max \xi - \min \xi \right)$ if
  $z^* = -\infty$ and, using the convention $(+ \infty) \cdot 0 = 0$,
  to
  \begin{align}
    \overline{L}_{\nu} (d)
    & = \Esp\left( L(\xi, d) \right)
      \nonumber\\
    & = \Esp\left( \left( \max \xi - z^*\right) \one_{\xi( x^* ) = z^*} \right)
      \;+\; \left(+\infty\right) \cdot \Prob\left( \xi( x^* ) \neq z^* \right)
      \nonumber\\
    & = \begin{cases}
      \Esp \left( \max \xi \right) - m_\nu(x^*),
      & \text{if } \sigma_\nu(x^*) = 0 ~ \text{and} ~ m_{\nu}(x^*) = z^*,\\
      +\infty & \text{otherwise,}
    \end{cases} \label{eq:loss:EI:bar}
  \end{align}
  if $z^* > -\infty$.
  The last equality follows from the simple
    observation that the event~$\left\{ \xi( x^* ) = z^* \right\}$ is
    almost sure if $\sigma_\nu(x^*) = 0$ and~$m_\nu(x^*) = z^*$, and
    negligible otherwise.

  In the case where there exists at least one~$x \in \Xset$ such
  that~$\sigma_\nu(x) = 0$, then it is clear that
  $\Esp \left( \max \xi \right) - m_\nu(x) < \Esp\left( \max \xi -
    \min \xi \right)$ for any such~$x$,
  and thus the expected loss is minimal for $d = (x^*, z^*)$ such that
  $x^* \in \argmax_{x: \sigma_\nu(x) = 0} m_\nu(x)$ and
  $z^* = m_\nu(x^*)$, which yields~\eqref{eq:EI:uncertainty-func}.
  In the case where $\sigma_\nu$ does not vanish, on the other hand,
  then~\eqref{eq:loss:EI:bar} is always infinite and thus the expected
  loss in minimal for any $d = (x^*, -\infty)$, which yields
  \eqref{eq:EI:uncertainty-func-2}.
  In both cases, we have proved that
  $\Hcal(\nu) = \min_{d \in \Dset} \overline{L}_{\nu} (d)$.

  We will now prove that there is no regular loss function~$L$ such
  that $\Hcal(\nu) = \min_{d \in \Dset} \overline{L}_{\nu} (d)$.
  To do so, we will show that $\Hcal$ cannot be
  decomposed as $\Hcal = \Hcal_0 + \Hcal_1$, with
  $\Hcal_0(\nu) = \int_\Sset L_0\, \dnu$ for some
  $L_0 \in \cap_{\nu \in \MGCX}\, \mathcal{L}^1 \left( \Sset, \Scal, \nu
  \right)$, and $\Hcal_1$ a \AscFgcd functional.

  Assume, for the sake of contradiction, that $\Hcal = \Hcal_0 + \Hcal_1$, with
  $\Hcal_0(\nu) = \int_\Sset L_0\, \dnu$ for some
  $L_0 \in \cap_{\nu \in \MGCX} \mathcal{L}^1 \left( \Sset, \Scal, \nu
  \right)$, and $\Hcal_1$ a \AscFgcd functional.
  Then, using the same martingale argument as in the proof of
  Theorem~\ref{thm:convergence:bis}, we have
  $\Hcal(P_n^{\xi}) \toas \Hcal(P_{\infty}^{\xi})$.
  Also, again by a martingale argument,
  $\Esp_{n}(\max \xi) \toas \Esp_{\infty}(\max \xi)$, %
  and therefore
  $\max_{\sigma_n(x) = 0} m_n(x) \toas \max_{\sigma_{\infty}(x) = 0}
  m_{\infty}(x)$.

  We will now show that this last convergence does not hold for a
  certain Gaussian process~$\xi$ on~$\Xset$, which yields a
  contradiction.
  For simplicity, we assume in the following that
  $\Xset = \left[ 0, 1 \right]$, but the same argument could be made
  on any $\Xset \subset \Rset^p$ that has a non-empty interior.

  Consider a Gaussian process~$\xi$ with mean $m(x) = x$ and
  covariance $k(x,y) = \exp(-(x-y)^2)$.
  Let~$(X_n)$ be a deterministic sequence, dense
  in~$\left[ 0, 1/3 \right]$.
  Then, as follows from the proof of Proposition~1 in
  \cite{vazquez2010pointwise}, we have $\sigma_{\infty}(x) = 0$ for
  all $x \in \left[0, 1 \right]$.
  Hence,
  $\max_{\sigma_{\infty}(x) = 0}\, m_{\infty}(x) = \max_{x \in [0,1]}
  \xi(x)$.
  Also, since $X_k \in [0,1/3]$ for all $k \in \mathbb{N}$, and since
  $\xi$~is a non-degenerate Gaussian process, we have
  $\max_{\sigma_n(x) = 0}\, m_n(x) \leq \max_{x \in [0,1/3]}\, m_n(x)$.
  This upper bound converges to $\max_{x \in [0,1/3]} \xi(x)$ almost
  surely,
  and thus $\max_{x \in [0,1/3]} \xi(x) = \max_{x \in [0,1]} \xi(x)$
  almost surely.
  This last equality cannot hold, however, because
  \begin{align*}
    \max_{x \in [0,1/3]} \xi(x)
    & \le \frac{1}{3} + \max_{x \in [0,\, 1/3]} (\xi(x) - x),\\
    \max_{x \in [2/3, 1]} \xi(x)
    & \ge \frac{2}{3} + \max_{x \in [2/3,\, 1]} (\xi(x) - x),
  \end{align*}
  and by symmetry $\max_{x \in [0,1/3]} (\xi(x) - x)$ and
  $\max_{x \in [2/3,1]} (\xi(x) - x)$ have the same distribution.
\end{proof}

Since the EI uncertainty functional does not derive from a regular
loss function, consistency cannot be proved using
Corollary~\ref{cor:RegLoss} as in the three previous examples.
The following result will thus be proved by a direct application of
the more general Theorem~\ref{thm:convergence:generic}.

\begin{theorem} \label{prop:EI:converges}
  For any quasi-SUR sequential design associated with $\Hcal$, as
  $n \to \infty$, almost surely and in $L^1$,
  $H_n \to 0$, $\max m_n \to \max \xi$ and $M_n \to \max\xi$.
\end{theorem}

\begin{proof}
  Since the uncertainty functional~$\Hcal$ derives from a loss
  function by Proposition~\ref{prop:EI:not:regular},
  it is DoA by Proposition~\ref{prop:H-infi-DoA} and, consequently,
  has the supermartingale property.

  Consider now a quasi-SUR sequential design associated with~$\Hcal$.
  Theorem~\ref{thm:convergence:generic} applies and therefore
  $\Gcal(P_n^\xi) \to 0$ almost surely. Observe also that, for all $n \ge 1$,
  \begin{equation*}
    M_{n+1} = \max_{\sigma_{n+1}(x) = 0} \xi(x)
    \geq
    \max \left( \xi\left(X_{n+1}\right),\, \max_{\sigma_n(x) = 0} \xi(x) \right)
    =
    \max \left( \xi\left(X_{n+1}\right) , M_n \right).
  \end{equation*}
  Hence
  $\Gcal(P_n^\xi) = \sup_{x \in \Xset} \Esp_{n, x} \left( M_{n+1} -
    M_n \right) \ge \sup_{x \in \Xset} \Esp_n\left( \max\left( 0,\,
      \xi(x) - M_n \right) \right)$, and thus
  \begin{equation} \label{eq:EI:gamme:to:zero}%
    \max_{x \in \Xset}\, \gamma\left(
      m_n(x) - M_n , \sigma_n^2(x)
    \right) \xrightarrow[n \to \infty]{\text{a.s.}} 0,
  \end{equation}
  where $\gamma$ denotes the function defined by
  $\gamma(a,b) = \Esp\left(\max\left(0,\, Z_{a,b}\right)\right)$,
  $Z_{a,b} \sim \Ncal(a,b)$.
  Recall from Section~3 in~\citet{vazquez2010convergence} that
  $\gamma$ is continuous and satisfies
  \begin{itemize}
  \item $\gamma (z, s^2) > 0$ if $s^2 > 0$,
  \item $\gamma (z, s^2) \ge z > 0$ if $z > 0$.
  \end{itemize}
  Recall also from Proposition~\ref{prop:conv-M:n-to-infty} that,
  almost surely, $m_n \to m_\infty$ and $\sigma_n \to \sigma_\infty$
  uniformly on~$\Xset$.
  Therefore we have
  \begin{equation}
    \label{eq:EI:gamma:is:zero}
    \forall x \in \Xset, \qquad
    \gamma\left( m_\infty(x) - M_\infty, \sigma_\infty^2(x) \right) = 0,
  \end{equation}
  almost surely, where $M_\infty$ denotes the almost sure limit of the
  increasing sequence~$(M_n)$, with $M_\infty \le \max \xi < +\infty$.
  (To see that~$(M_n)$ is increasing, observe that the set of
  points~$x \in \Xset$ such that~$\sigma_n(x) = 0$ is growing
  with~$n$, since $\left( \sigma_n(x) \right)$ is decreasing for
  any~$x$.)
  Considering the properties of $\gamma$, it follows
  from~\eqref{eq:EI:gamma:is:zero} that %
  almost surely, for all~$x \in \Xset$, %
  $\sigma_{\infty}(x) = 0$ and $ m_\infty(x) - M_\infty \leq 0$.
  Therefore, almost surely, we have $\xi = m_\infty$ and
  $M_\infty \ge \max m_\infty$.
  Since it is clear that $M_n \le \max m_n$ for all~$n$, we also have
  $M_\infty \le \max m_\infty$ in the limit, hence
  $M_\infty = \max m_\infty = \max \xi$ almost surely.

  We have proved so far that $\max m_n \to \max \xi$
  and~$M_n \to \max \xi$ almost surely.
  Moreover, $\Esp_n(\max \xi)$ is a martingale that converges almost
  surely and in~$L^1$ to~$\Esp_\infty(\max \xi) = \max \xi$ (see,
  e.g., Theorem~6.23 in~\cite{kall02}), and therefore
  $H_n = \Esp_n(\max \xi) - M_n \to 0$ almost surely.

  We conclude the proof by observing that all three convergence
  results also hold in the $L^1$ sense by the dominated convergence
  theorem.
\end{proof}

Finally, we remark that Proposition~\ref{prop:EI:converges} improves
the consistency result of \cite{vazquez2010convergence}, since it does
not impose the no-empty-ball property on the covariance function $k$.
Hence, Proposition~\ref{prop:EI:converges} also holds with very smooth
Gaussian processes, %
such as Gaussian processes with a Gaussian (a.k.a. squared
exponential) covariance function, %
or with Gaussian processes whose sample paths have symmetry
properties \cite{ginsb:sym}.

\appendix
\section{Technical results and proofs}

\subsection{Measurability results} \label{app:measurability}

\begin{lemma}
  \label{lem:P-measurability-2} %
  Let $\left( \Eset, \Ecal \right)$ denote a measurable space. %
  Let $\varphi:\Sset \times \Eset \to \RplusExt$ denote an
  $\Scal \otimes \Ecal$-measurable function.
  Then the function $\MGCX \times \Eset \to \RplusExt$,
  $(\nu, y) \mapsto \int_\Sset \varphi(f, y)\, \nu(\df)$, is
  $\McalGCX \otimes \Ecal$-measurable.
\end{lemma}

\begin{proof}
  The result is clear for any $\varphi = \one_{A \times B}$, with
  $A \in \Scal$ and~$B \in \Ecal$.
  Indeed,
  $\int_\Sset \varphi(f, y)\, \nu(\df) = \pi_A(\nu)\, \one_B(y)$,
  where $\pi_A$ denotes the evaluation map~$\nu \mapsto \nu(A)$, and
  the restriction of $\pi_A$ to~$\MGCX$ is $\McalGCX$-measurable.
  It can be extended to any $\varphi = \one_\Gamma$, with
  $\Gamma \in \Scal \otimes \Ecal$, using a standard monotone class
  argument, and then to any $\Scal \otimes \Ecal$-measurable function
  by linearity and increasing approximation by simple functions.
\end{proof}

In the following lemma, the Banach space $\Ccal(\Xset \times \Xset)$
is endowed with its Borel $\sigma$-algebra.

\begin{lemma}
  \label{lem:mean-var-meas}
  The mappings $m_{\bullet}: \MGCX \to \Sset$, $\nu \mapsto m_\nu$ and
  $k_\bullet: \MGCX \to \Ccal(\Xset \times \Xset)$, $\nu \mapsto k_\nu$ are
  measurable.
\end{lemma}

\begin{proof}
  The mapping~$m_\bullet$ is measurable if, and only if,
  $\nu \mapsto \varphi(m_\nu)$ is measurable for
  all~$\varphi \in \Sset'$ \citep[see,
  e.g.,][Theorem~2.2]{vakhania87}.
  Let $\varphi \in \Sset'$: there exists a unique signed
  measure~$\mu_\varphi$ on~$\Xset$ such that
  $\varphi(f) = \int_\Xset f\, \dmu_\varphi$.  It is then easy to
  check with Fubini's theorem that
  $\varphi(m_\nu) = \int \varphi(f)\, \nu(\df)$, and the conclusion
  follows from Lemma~\ref{lem:P-measurability-2}.
  The measurability of~$k_\bullet$ is established in a similar way,
  working on~$\Xset \times \Xset$ instead of~$\Xset$.
\end{proof}

Let $\Theta \subset \Sset \times \Ccal(\Xset \times \Xset)$ denote the
range of~$\Psi = \left( m_\bullet, k_\bullet \right)$, and let $\Tcal$
denote the trace on~$\Theta$ of the Borel $\sigma$-algebra
of~$\Sset \times \Ccal(\Xset \times \Xset)$.

\begin{lemma} \label{lem:meas-psi}
  $\Psi$ is a bi-measurable mapping from~$(\MGCX, \McalGCX)$
  to~$(\Theta, \Tcal)$.
\end{lemma}

\begin{proof}
  The measurability of~$\Psi$ follows from
  Lemma~\ref{lem:mean-var-meas}.
  Since $\McalGCX$ is generated by the evaluation maps (see
  Section~\ref{sec:Gauss-RE-RM}), $\Psi^{-1}$ is measurable if, and
  only if, $(m, k) \mapsto \left[ \GP(m,k) \right](A)$ is measurable
  for all~$A \in \Scal$.
  This is easily checked for any finite intersection of the form
  $A = \cap_k \left\{ f \in \Sset \mid f(x_k) \in \Gamma_k \right\}$,
  where $(x_k) \in \Xset^n$ and~$(\Gamma_k) \in \Bcal(\Rset)^n$.
  The result extends to the ball $\sigma$-algebra~$\Scal_0$ using a
  standard monotone class argument, which concludes the proof since
  $\Scal_0 = \Scal$ \citep[see, e.g.,][]{billi99}.
\end{proof}

\begin{proposition} \label{prop:EI:H:measurable}
  The expected improvement
  functional~\eqref{eq:EI:uncertainty-func}--\eqref{eq:EI:uncertainty-func-2}
  is $\McalGCX$-measurable.
\end{proposition}

\begin{proof}
  Let $\{ x_i \}$ denote a countable dense subset of~$\Xset$ and set,
  for all $k > 0$,
  \begin{equation*}
    \Hcal_k(\nu) = \int_\Sset \left( \max f - \min f\right)\, \nu(\df)
    - \sup_i \left( m_\nu(x_i) - \int_\Sset \min f\, \nu(\df) \right)\,
    \one_{\sigma_\nu(x_i) \le \frac{1}{k}}.
  \end{equation*}
  The mappings $\nu \mapsto \int_\Sset \max f\, \nu(\df)$,
  $\nu \mapsto \int_\Sset \min f\, \nu(\df)$,
  $(\nu, x) \mapsto m_\nu(x)$ and $(\nu, x) \mapsto \sigma_\nu^2(x)$
  are measurable by Lemma~\ref{lem:P-measurability-2}.
  As a consequence, for any~$k > 0$, the functional~$\Hcal_k$ is
  $\McalGCX$-measurable.
  The result follows from the fact that~$\Hcal_k \to \Hcal$ pointwise
  as $k \to \infty$.
\end{proof}

\subsection{The conditioning operator}
\label{app:conditioning}

Let $\ZZ_n = (Z_1, \ldots, Z_n)$ and~$\XX_n = (X_1, \ldots, X_n)$.
For any $(m, k) \in \Theta$, $\xx_n \in \Xset^n$
and~$\zz_n \in \Rset^n$, it is well known that the conditional mean and
covariance functions of~$(\xi(x))_{x \in \Xset}$
given~$\ZZ_n = \zz_n$, assuming a deterministic design $\XX_n = \xx_n$
(see Section~\ref{sec:bkgrnd:model}), are given by
\begin{align}
  m_n(x;\, \xx_n, \zz_n)
  & = m(x) + k(x,\xx_n)\, K(\xx_n)^\dagger \left( \zz_n - m(\xx_n) \right)
    \label{equ:mn}\\
  k_n(x, y;\, \xx_n)
  & = k(x, y) - k(x, \xx_n)\, K(\xx_n)^\dagger\, k(\xx_n, y),
    \label{equ:kn}
\end{align}
where $K(\xx_n)^\dagger$ denotes the Moore-Penrose pseudo-inverse of
$K(\xx_n) = \left( k(x_i, x_j) + \tau(x_i)^2\, \delta_{i,j} \right)_{1 \le
  i, j \le n}$, and $k(\xx_n, \cdot\,)$ and the other notations should be
self-explanatory.

\begin{lemma} \label{lem:meas-kappa-tilde}
  $\tilde\kappa_n: \left( \xx_n, \zz_n, (m,k) \right) \mapsto
  (m_n(\,\cdot\,;\, \xx_n, \zz_n),\, k_n(\,\cdot\,;\, \xx_n))$ is a
  measurable mapping from $\Xset^n \times \Rset^n \times \Theta$
  to~$\Theta$, where $\Theta$ is endowed with the
  $\sigma$-algebra~$\Tcal$ defined in the preceding section.
\end{lemma}

\begin{proof}
  First observe that for any~$\xx_n$, $k_n(\,\cdot\,;\, \xx_n)$ is the
  covariance function of $\xi - m_n(\,\cdot\,; \xx_n, \ZZ_n)$, which
  is a Gaussian process with continuous sample paths.
  Thus, $(m_n(\,\cdot\,;\, \xx_n, \zz_n),\, k_n(\,\cdot\,;\, \xx_n))$
  is indeed an element of~$\Theta$.
  The result then follows from the continuity of
  $(m, x) \mapsto m(x)$, $(k, x) \mapsto k(x,\,\cdot\,)$,
  and~$(k, x, y) \mapsto k(x,y)$, and the measurability of
  $K \mapsto K^\dagger$ \citep{schonfeld1973}.
\end{proof}

\begin{proof}[Proof of Proposition~\ref{prop:Cond-map}]
  Let $\kappa_n: \Xset^n \times \Rset^n \times \MGCX \to \MGCX$ denote
  the mapping defined by
  \begin{equation}
    \kappa_n(\xx_n, \zz_n, \nu) %
    = \GP\left(
      m_n(\,\cdot\,;\, \xx_n, \zz_n),\, %
      k_n(\,\cdot\,;\, \xx_n)
    \right),
  \end{equation}
  where $\nu = \GP(m, k) \in \MGCX$.
  Observe that, using the notations introduced in the previous
  section,
  $\kappa_n(\xx_n, \zz_n, \nu) = \Psi^{-1} \left( \tilde\kappa_n
    (\xx_n, \zz_n, \Psi(\nu)) \right)$: thus, it follows from
  Lemmas~\ref{lem:meas-psi} and~\ref{lem:meas-kappa-tilde} that
  $\kappa_n$~is measurable.
  Standard algebraic manipulations then show that
  \begin{equation*}
    \kappa_{n+m} \left(\xx_{n+m}, \zz_{n+m}, \nu \right)
    = \kappa_m \left( \xx_{n+1:n+m},\, \zz_{n+1:n+m},\,
      \kappa_n \left(\xx_n, \zz_n, \nu \right) \right),
  \end{equation*}
  whence it is easy to prove recursively that
  $\Prob_n^\xi \eqdef \kappa_n\left( \XX_n, \ZZ_n, \Prob^\xi \right)$
  satisfies the property
  $\Esp\left( U\, \Prob_n^\xi(\Gamma) \right) = \Esp\left( U\,
    \one_{\xi \in \Gamma} \right)$ for any sequential
  design~$\left( X_i \right)$, any $\Fcal_n$-measurable $U$ of the
  form $U = \Pi_{i=1}^n \varphi_i(Z_i)$ and any $\Gamma \in \Scal$ of
  the form
  $\Gamma = \cap_{j=1}^J \left\{ \xi(\tilde x_j) \in \Gamma_j
  \right\}$, with $\tilde x_j \in \Xset$, $\Gamma_j \in \Bcal(\Rset)$,
  $1 \le j \le J$.
  The result extends to any $\Fcal_n$-measurable $U$ and
  any~$\Gamma \in \Scal$ thanks to a monotone class argument, which
  proves that $\Prob_n^\xi$ is a conditional distribution of~$\xi$
  given $\Fcal_n$.
  Proposition~\ref{prop:Cond-map} is thus established with
  $\Cond_{x_1, z_1, \ldots, x_n, z_n}: \nu \mapsto \kappa_n \left(
    \xx_n, \zz_n, \nu \right)$.
\end{proof}

\begin{proposition}
  \label{prop:Jcalx:meas}
  The mapping $(x, \nu) \mapsto \Jcal_x(\nu)$ is
  $\Bcal(\Xset) \otimes \McalGCX$-measurable.
\end{proposition}

\begin{proof}
  Observe that $\Jcal_x(\nu)$ can be rewritten as
  \begin{equation}
    \label{equ:Jcalx:rewritten}
    \Jcal_x(\nu)  = \int_{\Rset} \Hcal
    \left( \kappa_1
      \left(
        x,\, m_\nu(x) + v\, s_\nu(x),\, \nu
      \right)
    \right)\, \phi(v)\, \dv,
  \end{equation}
  where $s_\nu^2 = k_\nu(x, x) + \tau^2 (x)$ and~$\kappa_1$ is defined
  as in the proof of Proposition~\ref{prop:Cond-map}.
  Using Lemma~\ref{lem:meas-psi} and the measurability of~$\kappa_1$,
  the integrand in the right-hand side of~\eqref{equ:Jcalx:rewritten}
  is easily seen to be a
  $\Bcal(\Xset) \otimes \Mcal \otimes \Bcal(\Rset)$-measurable
  function of~$(x, \nu, v)$.
  The result follows from Fubini's theorem.
\end{proof}

\begin{remark}
  As a consequence of Proposition~\ref{prop:Jcalx:meas},
  $J_n: x \mapsto \Jcal_x (\Pnxi)$ is an $\Fcal_n$-measurable process
  for all~$n$, and thus
  $J_n(X)$ is a well-defined $\Fcal_n$-measurable random variable for
  any $\Fcal_n$-measurable $\Xset$-valued random variable~$X$.
\end{remark}

\subsection{Convergence in $\MGCX$}
\label{app:ASCSCD:proofs}

\begin{proof}[Proof of Proposition~\ref{prop:conv-M:n-to-infty}]

  Recall from Proposition~\ref{prop:Cond-map} that the conditional
  distribution of~$\xi$ given~$\Fcal_n$ is of the
  form~$\Prob_n^\xi = \GM (m_n, k_n )$.
  Moreover, $\xi$ is a Bochner-integrable $\Sset$-valued random
  element: indeed, it is measurable (see, e.g., \cite{vakhania87}) %
  and $\lVert \xi \rVert_\infty$ is integrable (see, e.g., Theorem~2.9
  in~\cite{azais09level}).
  The conditional expectation $\Esp\left( \xi \mid \Fcal_n \right)$
  of~$\xi$ given~$\Fcal_n$ is thus well defined as an $\Sset$-valued
  random element (since $\Sset = \CcalX$ is a separable Banach space;
  see, e.g., Theorem~5.1.12 in~\cite{stroock}) and is easily seen to
  coincide with~$m_n$.
  As a consequence, it follows from Theorem~6.1.12 in~\cite{stroock}
  that $m_n$ converges uniformly, almost surely and
  in~$L^1\left(\Omega, \Fcal, \Prob \right)$, to
  $m_\infty \eqdef \Esp\left( \xi \mid \Fcal_\infty \right)$.
  The limit $m_\infty$ is, by definition of the conditional
  expectation, an $\Fcal_\infty$-measurable random element
  in~$\Sset$.

  Let us now prove that the sequence~$k_n$ converges uniformly to a
  continuous function~$k_\infty$.
  Since $\Prob_n^\xi = \Cond_{X_1, Z_1, \ldots, X_n, Z_n}(\Prob^\xi)$
  by Proposition~\ref{prop:Cond-map}, and since the sequence of
  conditional covariance functions depends only on the design
  points~$X_i$ (not on the observed values~$Z_i$), we can reduce
  without loss of generality to the case of a deterministic design
  ($X_i = x_ i \in \Rset$, for all $i \in \Nset$) and consider the
  associated \emph{deterministic} sequence~$(k_n)$.
  Let $\mu = \sum_{i = 1}^p \mu_i \delta_{\tilde x_i}$ denote any
  finitely supported measure on~$\Xset$, and let
  $\sigma_n^2(\mu) = \sum_{i,j=1}^p \mu_i \mu_j k_n (\tilde x_i,
  \tilde x_j)$ denote the conditional variance
  of~$Z = \sum_{i = 1}^p \mu_i \xi(\tilde x_i)$ given~$\Fcal_n$.
  Because $Z$ and the observations are jointly Gaussian, the sequence
  $\left( \sigma^2_n(\mu) \right)_{n \ge 1}$ is decreasing and
  therefore converges to a limit~$\sigma_\infty^2(\mu)$, for
  all~$\mu$.
  Thus,
  \begin{equation*}
    k_n(x, y) = \frac{1}{4} \left( %
      \sigma_n^2 \left( \delta_x + \delta_y \right)
      - \sigma_n^2 \left( \delta_x - \delta_y \right)
    \right) %
    \;\xrightarrow[n \to \infty]{}\;
    \frac{1}{4} \left( %
      \sigma_\infty^2 \left( \delta_x + \delta_y \right)
      - \sigma_\infty^2 \left( \delta_x - \delta_y \right)
    \right),
  \end{equation*}
  which proves convergence to a limit~$k_\infty(x, y)$.
  Moreover, we have for any $x, y, x',
  y' \in \Xset$:
  \begin{align}
    \left| k_n\left(x,y\right) - k_n\left(x',y'\right) \right|
    & \le \sigma_n(\delta_x)\, \sigma_n\left( \delta_y - \delta_{y'} \right)
    + \sigma_n(\delta_{y'})\, \sigma_n\left( \delta_x - \delta_{x'} \right)\\
    & \le \sigma_0(\delta_x)\, \sigma_0\left( \delta_y - \delta_{y'} \right)
    + \sigma_0(\delta_{y'})\, \sigma_0\left( \delta_x - \delta_{x'} \right).
  \end{align}
  Letting~$n$ go to~$+\infty$ in the left-hand side, we conclude
  that~$k_\infty$ is continuous.
  To see that the convergence $k_n \to k_\infty$ is uniform, consider
  the sequence of functions~$\Xset^2 \to \Rset$,
  $\left( x, y \right) \mapsto \sigma^2_n \left(\delta_x + \delta_y
  \right)$.
  This is a decreasing sequence of continuous functions, which
  converges pointwise to the continuous function
  $\left( x, y \right) \mapsto \sigma^2_\infty \left(\delta_x +
    \delta_y \right)$.
  Since $\Xset^2$ is compact, the convergence is uniform by Dini's
  first theorem.
  The same argument applies to
  $\left( x, y \right) \mapsto \sigma^2_n \left(\delta_x - \delta_y
  \right)$ and therefore to~$k_n$ by polarization.

  Finally, let $\Qrob$ denote any conditional distribution of~$\xi$
  given~$\Fcal_\infty$.
  We will prove that the $\Fcal_\infty$-measurable random measure
  $\Qrob$ is almost surely a Gaussian measure.
  Let $x \in \Xset$ and let $\phi_x$ denote the (random)
  characteristic function of~$\Qrob \circ \delta_x^{-1}$.
  It follows from Theorem~6.23 in~\citet{kall02} that, for all
  $u \in \Rset$,
  $ \phi_x(u) = \Esp_\infty \left( e^{iu\xi(x)} \right)
  \stackrel{\as}{=} \lim_{n \to \infty} \Esp_n \left(
    e^{iu\xi(x)} \right)$.
  Since
  \begin{equation*}
    \Esp_n \left(
      e^{iu\xi(x)} \right)
    = e^{ium_n(x)}\, e^{-\frac{1}{2}k_n(x,x)u^2}
    \xrightarrow[n\to\infty]{\as}
    e^{ium_\infty(x)}\, e^{-\frac{1}{2}k_\infty(x,x)u^2},
  \end{equation*}
  we conclude from the continuity of~$\phi_x$ and Levy's theorem that
  $\Qrob \circ \delta_x^{-1} = \Ncal \left( m_\infty(x),
    k_\infty(x,x) \right)$ almost surely.
  The argument extends to any image measure of the
  form~$\Qrob \circ h^{-1}$, with
  $h = \left( \delta_{y_1}, \ldots, \delta_{y_m} \right)$.
  Considering first the case where the $y_j$'s are taken in a
  countable dense subset of~$\Xset$ and then using the continuity of
  the elements of~$\Sset$, we conclude that there is an almost sure
  event $\Omega_0 \in \Fcal_\infty$ such that, for
  $\omega \in \Omega_0$, $\left( \delta_x \right)_{x \in \Xset}$ is a
  Gaussian process defined on the probability space
  $\left( \Sset, \Scal, \Qrob(\omega, \cdot) \right)$, %
  and thus $\Qrob(\omega, \cdot)$ is a Gaussian measure for all
  $\omega \in \Omega_0$.
  Finally, letting
  \begin{equation*}
    \Prob_\infty^\xi(\omega, \cdot) =
    \begin{cases}
      Q(w, \cdot) & \text{if } w \in \Omega_0,\\
      \GP(0, 0) & \text{otherwise,}
    \end{cases}
  \end{equation*}
  we have constructed an $\Fcal_\infty$-measurable random element
  in~$\MGCX$ such that $\Prob_n^\xi \to \Prob_\infty^\xi$ a.s.\ for
  the topology introduced in Definition~\ref{def:topology:Mzero},
  thereby concluding the proof.
\end{proof}

\begin{proof}[Proof of Proposition~\ref{prop:conv-M:cond-pt}]

  Let $\nu = \GP(m, k) \in \MGCX$ and let
  $(x_j, z_j) \to (x_\infty, z_\infty)$ in~$\Xset \times \Rset$.
  For any $j \in \Nset \cup\{+\infty\}$, we have
  $\Cond_{x_j,\, z_j} (\nu) = \GP(m_1(\,\cdot\,;\, x_j, z_j),\,
  k_1(\,\cdot\,;\, x_j)$, where~$m_1$ and~$k_1$ are given
  by~\eqref{equ:mn}--\eqref{equ:kn}.
  It is then easy to check that $m_1(\,\cdot\,;\, x_j, z_j)$
  and~$k_1(\,\cdot\,;\, x_j)$ converge uniformly
  to~$m_1(\,\cdot\,;\, x_\infty, z_\infty)$
  and~$k_1(\,\cdot\,;\, x_\infty)$, respectively,
  using the facts that $k$~is uniformly continuous
  over~$\Xset \times \Xset$ (since $k$~is continuous
  and~$\Xset \times \Xset$ is a compact metric space) and that
  $K \mapsto K^\dagger$ is continuous at
  $K = k_\nu(x, x) + \tau^2(x) > 0$ (the covariance matrix is actually
  a scalar in this case).
\end{proof}

\subsection{Existence of SUR and quasi-SUR sequential designs}
\label{sec:SUR:existence}

This section contains general existence results for
$\varepsilon$-quasi-SUR sequential designs.
Recall that $\Xset$ is assumed, throughout the paper, to be a compact
metric space (see Standing assumptions~\ref{ass:standing:assumption}).

\begin{theorem} \label{thm:Jn-cont:and:SUR-exist}
  Let the assumptions of Theorem~\ref{thm:convergence:bis} hold.
  Then,
  \begin{enumerate}[ a)]
  \item\label{assert:Jn-continuous} for any sequential design, the
    sample paths of~$J_n$ are continuous
    on~$\left\{ x \in \Xset: s_n^2(x) > 0 \right\}$;
  \item\label{assert:SUR-strat-exists} for any sequence
    $\varepsilon = \left( \varepsilon_n \right)$ of strictly positive real
    numbers, there exists an $\varepsilon$-quasi-SUR sequential
    design~$\left( X_n \right)_{n\ge 1}$ associated with~$\Hcal$.
  \end{enumerate}
\end{theorem}

\begin{proof}
  We will assume without loss of generality that $\Hcal_0 = 0$,
  since~$\Hcal_0$ only adds a constant term (i.e., a term that does
  not depend on~$x$) to the value of the sampling criterion.

  %
  %
  Let us first prove Assertion~\eqref{assert:Jn-continuous}. %
  Since $J_n(x) = \Jcal_x\bigl( \Pnxi \bigr)$, it is equivalent to
  prove that the result holds at~$n = 0$ for
  any~$\Prob_0^\xi \in \MGCX$. %
  Assume then that~$n = 0$, fix $x \in \Xset$ such that
  that~$s_0^2(x) = k(x,x) + \tau^2(x) > 0$, %
  and let $\left(x_j\right)$ denote a sequence in~$\Xset$ such
  $x_j \to x$. %
  Recall from~\eqref{equ:Jn-def-rigoureuse} that
  $J_0(x) = \Jcal_x(\Prob_0^\xi) = \Esp\bigl( \Hcal \bigl( \Cond_{x,\,
    Z_1(x)} \bigl( \Prob_0^\xi \bigr) \bigr) \bigr)$. %
  Set $\nug_k = \Cond_{x_k,\, Z_1(x_k)} \bigl( \Prob_0^\xi \bigr)$ and
  $\nug_\infty = \Cond_{x,\, Z_1(x)} \bigl( \Prob_0^\xi \bigr)$. %
  We have $\nug_k \in \Pfrak(\xi)$ for all
  $n \in \Nset \cup \{+\infty\}$, and $\nug_k \to \nug_\infty$ by
  Proposition~\ref{prop:conv-M:cond-pt}. %
  It follows that $\Hcal (\nug_k) \toas \Hcal (\nug_\infty)$ since
  $\Hcal$ is \AscFgcd, and thus
  $\Jcal_{x_k}(\Prob_0^\xi) = \Esp\bigl( \Hcal(\nug_k) \bigr) \to
  \Esp\bigl( \Hcal(\nug_\infty) \bigr) = \Jcal_x(\Prob_0^\xi)$ since
  $\left( \Hcal(\nug_k) \right)$ is uniformly integrable.
  Assertion~\eqref{assert:Jn-continuous} is proved.

  %
  %
  Consider now the following compact subsets of~$\Xset$:
  \begin{align}
    B_{n,\gamma}(\omega) %
    &= \left\{ x \in \Xset \mid\, s_n\left( \omega, x \right)
      \ge \gamma^{-1} > 0 \right\},\\
    A_{n,\gamma}(\omega) %
    &= B_{n, \gamma} (\omega) \,\cap\, %
      \left\{ x \in \Xset \mid\, J_n \left( \omega, x \right) \le
      \inf J_n(\omega, x) + \varepsilon_n \right\}.
  \end{align}
  Let us prove that, on the event~$\left\{ s_n \not\equiv 0 \right\}$,
  the set~$A_{n, \gamma}(\omega)$ is non-empty for large values
  of~$\gamma$.
  Assume that $s_n(\omega, \cdot) \not\equiv 0$,
  and recall that $J_n\left( \omega, \cdot \right) \le H_n(\omega)$
  by~\eqref{equ:Jn-smaller-than-Hn}.
  If $J_n(\omega, \cdot) \equiv H_n$, then for any~$x$ such
  that~$s_n(\omega, x) > 0$ and any $\gamma \ge s_n(\omega, x)^{-1}$
  we have $x \in A_{n,\gamma}(\omega) = B_{n,\gamma}(\omega)$.
  If $\inf_x J_n(\omega, x) < H_n$, pick a sequence~$(x_k)$ such that
  $J_n(\omega, x_k) \to \inf_x J_n(\omega, x)$.
  For some $k$ large enough,
  $J_n(\omega, x_k) \le \inf_x J_n(\omega, x) + \varepsilon_n$ and
  $J_n(\omega, x_k) < H_n$.
  As a consequence, $s_n(\omega, x_k) > 0$ (this follows
  from~\eqref{equ:Jn-2eme-ecriture-intsimple} and the fact that
  $\Cond_{x, m_\nu(x)} (\nu) = \nu$ if $s^2_\nu(x) = 0$) and thus
  $x_k \in A_{n, \gamma}(\omega)$ for any
  $\gamma \ge s_n(\omega, x_k)^{-1}$.
  In both cases the claim is proved.

  Since $\Xset$~is a compact metric space, it is easily proved that
  $\omega \mapsto A_{n,\gamma}(\omega)$ is an $\Fcal_n$-measurable
  \RCS, and thus admits \citep[see,
  e.g.,][Theorem~2.13]{molchanov06} an $\Fcal_n$-measurable
  selection~$X_{n+1}^{(\gamma)}$, i.e., an $\Xset$-valued random
  variable such that $X_{n+1}^{(\gamma)} \in A_{n, \gamma}$ on the
  event~$\{ A_{n, \gamma} \neq \varnothing \}$.
  Let $\tilde x$ denote an arbitrary fixed point in~$\Xset$.  Setting
  \begin{equation}
    X_{n+1} = \left\{
      \begin{aligned}
        & \tilde x \quad %
        && \text{if } s_n \equiv 0,\\
        & X_{n+1}^{(k)} \quad %
        && \text{if } A_{n,k} \neq \varnothing %
        \text{ and } A_{n,l} = \varnothing,\; \forall l < k,\\
      \end{aligned}
    \right.
  \end{equation}
  provides the desired $\varepsilon$-quasi-SUR strategy and thus
  completes the proof.
\end{proof}

In some situations, it is possible to prove directly the continuity of
the sampling criteria~$J_n$ on the whole of~$\Xset$ (see
Section~\ref{sec:example:ei} for an example), in which case a stronger
existence result can be formulated as
in~\cite{vazquez2010convergence}, that does not even require the
supermartingale property:

\begin{theorem} \label{thm:Jn-cont:and:SUR-exist:VARIANT}
  Let $\Hcal$ denote a measurable uncertainty functional on~$\MGCX$,
  such that, for all $\nu \in \MGCX$, $x \mapsto \Jcal_x(\nu)$ is
  finite and continuous on~$\Xset$.
  Then,
  \begin{enumerate}[ a)]
  \item for any sequential design, the
    sample paths of~$J_n$ are continuous on~$\Xset$;
  \item there exists a SUR sequential
    design~$\left( X_n \right)_{n\ge 1}$ associated with~$\Hcal$.
  \end{enumerate}
\end{theorem}

\begin{proof}
  Assertion a) follows trivially from the fact that
  $J_n(x) = \Jcal_x\bigl(\Prob_n^\xi\bigr)$, and a SUR sequential
  design is again obtained using the measurable selection theorem for
  \RCSs.
\end{proof}

\subsection{Miscellaneous}

\begin{lemma}
  \label{lem:ortho-sum}
  Let $U$, $V$ and~$W$ be real-valued random variables such that
  \begin{enumerate}
  \item $W$ is independent of~$(U, V)$,
  \item $V$ and~$W$ are Gaussian.
  \end{enumerate}
  If $U$ is orthogonal to~$L^2(V+W)$, then $U$ is orthogonal
  to~$L^2(V)$.
\end{lemma}

\begin{remark}
  The reverse implication is also true, but not needed in the paper.
\end{remark}

\begin{proof}
  Assume without loss of generality that~$U$, $V$ and~$W$ are
  centered.
  Assume further that~$U$ is not orthogonal to~$L^2(V)$.
  Then, there exists a smallest integer~$k_0$ such that
  $\cov (U, V^{k_0}) \neq 0$.
  Indeed, we would have otherwise $\cov (U, H_k(V)) = 0$ for all~$k$,
  where $H_k$ denotes the $k^{\text{th}}$ Hermite polynomial, and thus
  $U$ would be orthogonal to~$L^2(V)$ since
  $\left( H_k(V) \right)_{k \in \Nset}$ is an orthonormal basis
  of~$L^2(V)$.
  Using that $\cov (U, V^k) = 0$ for all $k < k_0$, we have:
  \begin{equation}
    \cov\left( U, (V+W)^{k_0} \right)
    \;=\; \sum_{k=0}^{k_0} \binom{k_0}{k}\,
    \Esp\left( U V^k \right) \Esp\left( W^{k - k_0} \right)
    \;=\; \Esp\left( U V^{k_0} \right)
    \;\neq\; 0.
  \end{equation}
  Therefore $U$ is not orthogonal to~$L^2(V+W)$, which concludes the
  proof by contraposition.
\end{proof}

\section*{Acknowledgments}

Part of David Ginsbourger's contribution took place within the
framework of the ``Bayesian set estimation under random field priors''
project (Number 146354) funded by the Swiss National Science
Foundation.
Fran\c{c}ois Bachoc acknowledges support from the ``PEPITO'' project
funded by the French National Agency for Research.
The authors would like to thank Luc Pronzato for pointing out the
important connection with DeGroot's earlier work on uncertainty
functionals, Dario Azzimonti for proofreading,
and two anonymous reviewers for their careful reading and numerous
suggestions of improvement.

\bibliographystyle{apalike}
\bibliography{surconv-paper}

\end{document}